\documentclass[11pt]{article}

\PassOptionsToPackage{table}{xcolor}

\usepackage[preprint]{acl}

\usepackage{times}
\usepackage{latexsym}
\usepackage[T1]{fontenc}
\usepackage[utf8]{inputenc}
\usepackage{microtype}
\usepackage{inconsolata}
\usepackage{graphicx}

\usepackage{amsmath}
\usepackage{amssymb}
\usepackage{amsfonts}
\usepackage{amsthm}
\usepackage{mathtools}
\usepackage{bm}
\usepackage{nicefrac}
\usepackage{booktabs}
\usepackage{seqsplit} 
\newcommand{\bsha}[1]{\texttt{\seqsplit{#1}}}

\newtheorem{theorem}{Theorem}
\newtheorem{lemma}{Lemma}
\newtheorem{corollary}{Corollary}
\newtheorem{proposition}{Proposition}[section]
\newtheorem{assumption}{Assumption}[section]
\theoremstyle{definition}
\newtheorem{definition}{Definition}[section]
\theoremstyle{remark}
\newtheorem{remark}{Remark}[section]

\definecolor{ourbg}{HTML}{FFF6E0}
\newcommand{\dn}{\,$\downarrow$}
\newcommand{\up}{\,$\uparrow$}
\newcommand{\our}{\rowcolor{ourbg}}
\newcommand{\starours}{\textbf{\textcolor{orange!80!black}{$\star$}}\,}

\newenvironment{shaded}{}{}

\usepackage{cleveref}
\crefname{assumption}{Assumption}{Assumptions}
\Crefname{assumption}{Assumption}{Assumptions}

\title{QAM-W: Joint 2D Codebook Quantization for LLM Weights via Hadamard
Rotation and Activation-Aware Scaling}

\author{Preetam Sharma \\
  Independent Research \\
  \texttt{preetam@manifoldlab.ai}
  \And
  Kacper Dobek \\
  Institute of Computing Science \\
  Poznan University of Technology, Poznan, Poland\\
  \texttt{kacper.dobek@cs.put.poznan.pl}}

\begin{document}
\maketitle

\begin{abstract}

Scalar post-training quantizers discard pairwise coordinate structure
within weight rows. We introduce QAM-W (Quadrature Amplitude Modulation
for Weights), a codec that recovers this structure: each row is L2-normalized, block-Hadamard rotated, paired
into 2D coordinates, and quantized against a single Lloyd-Max codebook
trained on the unit circular Gaussian, with activation-aware
per-channel scaling. In a cross-model study spanning five LLMs from
four families (1.1B--13B parameters) and eight quantized
configurations, the activation-aware variant at $\approx 5.5$ bpw
stays within $\pm 0.4\%$ of BF16 WikiText-2 perplexity on every
model, matching the SmoothQuant W8A8 quality envelope at $32\%$ fewer
weight bits. Joint 2D coding outperforms polar (amplitude $\times$
phase) coding by 2--15~pp $\Delta$PPL at equal bitrate, and paired KL
against BF16 tracks $\Delta$PPL\% at Spearman $\rho = 0.99$ across
37 (method, model) rows, consistent with a monotone composite bound
from codec distortion to KL divergence. A 3.5~bpw variant is
competitive on quantization-tolerant architectures. At strict 4~bpw,
the rotated-codebook frontier method QTIP outperforms QAM-W; the
contribution is the quality-preserving 5--6~bpw band.

\end{abstract}

\section{Introduction}
\label{sec:introduction}


\begin{shaded}
Most post-training quantizers in current use are scalar:
each weight is mapped independently to a small integer index, with
per-group or per-channel scales. Scalar coding is simple but leaves
rate-distortion slack for sources whose coordinates are correlated.
Within a weight row, coordinates are typically correlated; after an
orthogonal rotation they approximate a 2D circular Gaussian.
Quantizing pairs jointly -- as points in the plane rather than as
two independent scalars -- recovers pair-level structure that a
scalar codebook discards.

The empirical claim of this work is bounded: \emph{in the
$\approx 5$--$6$ bpw quality-preserving regime}, an activation-aware
joint-2D codec stays within $\pm 0.4\%$ of BF16 perplexity across
five LLMs from four families spanning $1.1$B--$13$B parameters, matching
SmoothQuant~\citep{xiao2024smoothquant} W8A8's quality at
$\approx 32\%$ fewer \emph{weight} bits.
At strict $4$~bpw, QTIP~\citep{tseng2024qtip} outperforms QAM-W
(\cref{sec:frontier}); the contribution here is the 5--6~bpw band.

Rotation-based methods such as QuIP~\citep{chee2024quip} and
QuIP\#~\citep{tseng2024quip} apply incoherent rotations to
homogenize coordinate magnitudes before scalar or lattice coding.
Activation-aware methods such as AWQ~\citep{lin2024awq} exploit the
per-channel activation scale but leave joint coordinate structure
untouched. The codec presented here combines all three ideas: a
deterministic block-Hadamard rotation, joint 2D Lloyd-Max coding
against a single codebook trained on the unit circular
Gaussian~\citep{lloyd1982least}, and AWQ-style per-channel scaling
driven by the layer-output trace identity.

Throughout, reported bits-per-weight (bpw) refers to weight-storage 
and memory-bandwidth footprint. The current pipeline dequantizes to 
BF16 before matmul; latency claims would require a fused kernel and 
are out of scope (see \cref{sec:limitations}).

\paragraph{Contributions:}
\begin{itemize}
  \item \textbf{Codec and analysis.} A row-norm-factored,
        block-Hadamard-rotated, joint 2D Lloyd-Max codec for LLM
        weights with activation-aware scaling, backed by a composite
        monotone bound from codec distortion to KL divergence
        (\cref{sec:method,sec:analysis}).
  \item \textbf{Cross-model study.} Eight quantized configurations
        on five LLMs under a unified perplexity, paired-KL, and
        six-task harness protocol
        (\cref{sec:cross-model,sec:frontier}).
  \item \textbf{KL as rank diagnostic.} Across $37$ (method, model)
        rows, paired KL and $\Delta$PPL\% are rank-correlated at
        Spearman $\rho = 0.99$.
\end{itemize}
\end{shaded}

\section{Related Work}
\label{sec:related}


\begin{shaded}

\paragraph{Scalar post-training quantization.}
Round-to-nearest (RTN) quantization~\citep{rtn} with per-group scaling
is the simplest baseline. GPTQ~\citep{frantar2022gptq} compensates
quantization error column-by-column using approximate second-order
information. AWQ~\citep{lin2024awq} protects high-activation channels
via per-channel scaling; AutoRound~\citep{cheng2023autoround} refines
rounding with SignSGD; HQQ~\citep{badri2023hqq} treats rounding as a
half-quadratic problem without calibration data.
SmoothQuant~\citep{xiao2024smoothquant} redistributes quantization
difficulty from activations to weights through a diagonal scaling
transform, targeting joint W8A8. OmniQuant~\citep{shao2024omniquant}
jointly learns clipping and smoothing transforms.
LLM.int8()~\citep{dettmers2022int8} keeps outlier activation channels
in 16-bit mixed precision. SqueezeLLM~\citep{kim2024squeezellm}
couples a non-uniform scalar codebook with a dense-and-sparse
decomposition; QLoRA's NF4~\citep{dettmers2024qlora} places levels
on the Gaussian distribution.

\paragraph{Rotated and vector quantization.}
QuIP~\citep{chee2024quip} and QuIP\#~\citep{tseng2024quip} apply
random or structured orthogonal rotations to decorrelate weight
coordinates before quantization.
QuaRot~\citep{ashkboos2024quarot} extends rotation through the
residual stream for end-to-end 4-bit inference;
SpinQuant~\citep{liu2024spinquant} replaces the fixed rotation with a
learned one. QAM-W shares the rotation principle but opts for a
deterministic block-Hadamard, which requires no calibration on the
rotation itself. Iso-bpw comparison against the learned-rotation
family is deferred to future work.

On the VQ side, AQLM~\citep{egiazarian2024aqlm} uses additive
codebooks and pushes below $2.5$~bpw.
QTIP~\citep{tseng2024qtip} combines an $E_8$-lattice trellis with
Hadamard incoherence and is the strongest published rotated-codebook
method; it and QAM-W occupy different operating points, with QTIP
stronger at strict 4~bpw and QAM-W stronger in the 5--6~bpw band
(\cref{sec:frontier}).
VPTQ~\citep{liu2024vptq} stacks residual codebooks down to
$\sim\!2$~bpw; GPTVQ~\citep{vanbaalen2024gptvq} shows the
rate-distortion gain from higher codebook dimension saturates around
$d\!=\!8$, placing QAM-W's $d\!=\!2$ choice at the small end of that
spectrum by design.

\paragraph{Concurrent work.}
PolarQuant~\citep{vicentino2026polarquant} couples a Walsh-Hadamard
rotation with a \emph{scalar} ($d\!=\!1$) Lloyd-Max codebook, sharing
QAM-W's first two pipeline stages but differing in codebook dimension.
The $d\!=\!1$ vs.\ $d\!=\!2$ ablation is the highest-priority follow-up.

\paragraph{Signal-processing origins.}
Quadrature amplitude modulation (QAM) originates in digital
communications~\citep{proakis2008digital};
QAM-W adapts this geometric view to weight compression. The codebook
training follows the Lloyd-Max
tradition~\citep{lloyd1982least,max1960quantizing} extended to 2D
sources, and can be seen as a Hadamard-rotated generalization of the
$k$-means weight clustering in Deep
Compression~\citep{han2016deepcompression}.
\end{shaded}

\section{QAM-W Codec Pipeline}
\label{sec:method}


This section describes the codec as an encoder/decoder pair. The
mathematical properties of the rotation and polar baseline are
separated into \cref{sec:codec-analysis}.

\subsection{Encoder}

\begin{shaded}
QAM-W quantizes a weight matrix
$W \in \mathbb{R}^{d_{\mathrm{out}} \times d_{\mathrm{in}}}$ row-by-row.
For each row $w$, the encoder performs the following steps:
\end{shaded}
\begin{enumerate}
  \item \textbf{Separate scale from direction.} Record the row norm
        $r=\|w\|_2$ as separate metadata and normalize the row to
        $u=w/r$ when $r>0$. If $r=0$, the row is encoded as a zero row.
  \item \textbf{Rotate the direction.} Apply a deterministic sign-masked
        block-Hadamard rotation, giving $y=R_{\mathrm{fwd}}u$.
  \item \textbf{Pair coordinates.} Group adjacent rotated coordinates into
        complex values $z_k=y_{2k}+iy_{2k+1}$ for
        $k=0,\ldots,d/2-1$.
  \item \textbf{Normalize each pair.} Divide each pair by its calibrated
        scale $\sigma_k$.
  \item \textbf{Quantize.} Encode each normalized pair either with the polar
        baseline codebooks or with the joint 2D QAM-W codebook.
  \item \textbf{Pack.} Store the resulting codebook indices in a bit-packed
        byte stream.
\end{enumerate}

\subsection{Decoder}

\begin{shaded}
The decoder reverses the encoder: unpack codebook indices, look up
the corresponding centroids, rescale each pair by $\sigma_k$,
reassemble pairs into a rotated real vector, apply the inverse
rotation $R_{\mathrm{inv}}$, and multiply by the stored row norm~$r$.
The benchmark charges 16~bits per row for norm metadata; a fully
serialized f16 row norm would add a small radial quantization term
not part of the directional analysis in \cref{sec:codec-analysis}.
\end{shaded}

\subsection{Rotation Choice}

The rotation used by the encoder is a block-Hadamard transform with a
deterministic sign mask. It is cheap to apply with a butterfly algorithm and
requires no stored dense matrix. The implementation chooses the block size $b$
as the largest power of two dividing the row dimension $d_{\text{in}}$, capped at
$b_{\max}=1024$. For the weight matrices in the experiments in this paper,
$d_{\text{in}}=2048$ uses $b=1024$ and $d_{\text{in}}=5632$ uses $b=512$. The rotation's exact isometry property is proved in
\cref{lem:rotation}.

\subsection{Pair Calibration}
\label{subsec:complex}

After rotation, adjacent coordinates are modeled empirically as approximately
circular Gaussian pairs with pair-specific scale $\sigma_k$. Under the idealized
model that the two coordinates in pair $k$ are zero-mean Gaussian variables with
equal variance $\sigma_k^2$ and zero covariance, the amplitude $|z_k|$ is
Rayleigh distributed and $\mathbb{E}|z_k|=\sigma_k\sqrt{\pi/2}$. Calibration
therefore estimates
\begin{equation}
  \hat{\sigma}_k =
  \frac{\operatorname{mean}(|z_k|)}{\sqrt{\pi/2}}
\end{equation}
from up to 1024 unit-normalized rows per weight matrix.

\subsection{Polar Baseline}

The polar baseline quantizes amplitude and phase independently. The amplitude
$|z_k|/\sigma_k$ is quantized with an $N_a=2^{B_a}$ level Lloyd-Max codebook for
the unit Rayleigh density $f(r)=r e^{-r^2/2}$. The phase $\arg(z_k)$ is rounded
to one of $N_p=2^{B_p}$ uniform bins. The total per-pair budget is
$B=B_a+B_p$ bits. The pairwise distortion model for this baseline is analyzed in
\cref{thm:pairwise}.

\subsection{Joint 2D QAM}

The main QAM-W variant replaces the two independent polar codebooks with a
single 2D codebook
$\mathcal{C}=\{c_0,\ldots,c_{2^B-1}\}\subset\mathbb{R}^2$. The codebook is
trained by Lloyd iterations on the unit circular Gaussian
\begin{equation}
  f(x,y)=\frac{1}{2\pi}\exp\left(-\frac{x^2+y^2}{2}\right).
\end{equation}
Encoding maps each normalized pair $z_k/\sigma_k$ to its nearest codebook entry
by Euclidean distance. Decoding is a table lookup followed by multiplication by
$\sigma_k$.

The current experiments use $B=11$ bits per pair, matching the
$5{+}6$ polar configuration in nominal pair payload. Native byte alignment and stream layout
can differ between polar and joint-2D implementations; the comparison is
therefore equal-bits-per-pair rather than byte-identical storage.

The choice $d=2$ is also deliberate at the decode end: reconstruction is a
  single lookup into a $2^B$-entry table that, at $B \le 11$, occupies at most
  $16$~KB and stays L1-resident. This is structurally simpler than the
  multi-codebook lookups and vector additions of additive methods such as
  AQLM~\citep{egiazarian2024aqlm}, or the sequential trellis decoding of
  QTIP~\citep{tseng2024qtip}. Given that rate-distortion gains saturate by
  $d \approx 8$ (\cref{sec:related}), $d=2$ deliberately trades part of the
  asymptotic ceiling for a decode path that is one cache-resident table read.

\subsection{Bitrate Accounting}

The bitrate per weight includes $B$ bits per pair for the codebook index and
16 bits per row for row-norm metadata. Since each pair covers two scalar
coordinates,
\begin{equation}
  \operatorname{bpw} = \frac{B}{2} + \frac{16}{d_{\mathrm{in}}}.
\end{equation}
For $d_{\mathrm{in}}=2048$ and $B=11$, this is approximately 5.51 bpw before
byte-alignment overhead. Reported bitrates include the row-level packing and
alignment used by the implementation.

\subsection{The QAM-W-3.5 Low-Bit Variant}
\label{subsec:qam-w-7}

To probe the sub-4 bpw regime, QAM-W is instantiated with $B=7$ bits per pair,
giving a $2^7=128$-entry 2D Lloyd-Max codebook trained on the same unit
circular Gaussian. Activation-aware per-channel scaling (\cref{sec:activation})
remains applicable. For $d_{\mathrm{in}}=2048$ the per-pair budget plus row-norm
metadata gives $\operatorname{bpw}\approx 3.51$; the AWQ-aware variant adds
$\sim 0.003$ bpw of per-channel scale metadata for $\alpha=0.3$. This is the
configuration analyzed in \cref{sec:low-bit}.

\section{Analysis}
\label{sec:analysis}

QAM-W's design rests on four analytical results. We state them here and
defer the formal statements, proofs, and supporting empirical checks to
\cref{sec:codec-analysis,sec:activation,sec:model-behavior}.

\paragraph{The rotation is distortion-free.} The sign-masked
block-Hadamard rotation is an exact isometry in real arithmetic
(\cref{lem:rotation}): it adds no reconstruction error and only
redistributes energy within each block before quantization. It does not
by itself make the rotated coordinates Gaussian; that the post-rotation
complex pairs concentrate near a circular Gaussian is an empirical
property, checked by QQ plots against fitted Rayleigh and Gaussian
marginals (\cref{subsec:qq-check}).

\paragraph{Joint 2D coding dominates polar at fixed rate.} For the polar
baseline, the per-pair distortion splits into an amplitude term and a
phase term with a closed form under the Rayleigh--Lloyd source model
(\cref{thm:pairwise,cor:rayleigh-polar}). Because that factorization
codes amplitude and phase independently, it cannot capture the
pair-level correlation that a single 2D codebook exploits.

\paragraph{Activation-weighted error is the right objective.} The
layer-output error obeys the trace identity
$\|X\Delta W^\top\|_F^2 = \operatorname{Tr}(\Delta W \hat{M}_x \Delta W^\top)$
(\cref{prop:layer}), so what a quantized layer actually contributes is
\emph{activation-weighted} error, not raw weight MSE. This motivates
AWQ-style per-channel scaling, which attenuates the contribution of
high-RMS input channels under a diagonal-$M_x$ model
(\cref{cor:diagonal-rms,prop:aware}). Empirically, the amplification of
weight-domain Frobenius error into layer-output RMSE varies by
$0.55$--$5.5\times$ across layers and methods (\cref{tab:layer-diag}).

\paragraph{A monotone codec-to-KL bound.} Chaining the rotation isometry,
the pairwise identity, the layer-output trace identity, and a
softmax-Fisher local KL expansion (\cref{cor:kl-bridge}) yields an upper
bound on the paired KL between the BF16 and quantized models that is
monotone in two scalars: the per-pair Lloyd distortion $D_B$ (codec) and
a model-fixed constant $C_W$ (\cref{thm:composite}). The bound motivates,
but does not prove, that realized KL should rank methods the same way
$D_B$ does; across the $37$ (method, model) rows in this study, paired KL
and $\Delta$PPL\% are rank-correlated at Spearman $\rho = 0.99$
(\cref{fig:kl-vs-dppl}).

\section{Experiments}
\label{sec:experiments}


\begin{shaded}
Eight quantized configurations are evaluated against a BF16 reference
on three instruction-tuned LLMs:
TinyLlama-1.1B-Chat~\citep{zhang2024tinyllama},
Qwen2.5-3B-Instruct~\citep{yang2024qwen25}, and
Mistral-7B-Instruct-v0.3~\citep{jiang2023mistral}.
These are dense decoder-only transformers with standard attention,
chosen so the comparison rests on mature, audited loading and
evaluation tooling (see Limitations).
Only MLP gate, up, and down projections are quantized; attention
weights and the LM head remain BF16.

\paragraph{Baselines.} Five external methods: RTN~\citep{rtn},
GPTQ~\citep{frantar2022gptq}, AWQ~\citep{lin2024awq}, and
AutoRound~\citep{cheng2023autoround}, all at W4A16 g128
($\approx 4.1$~bpw); SmoothQuant~\citep{xiao2024smoothquant} W8A8
($\approx 8.1$~bpw). \emph{SmoothQuant eval mode:} weights are
smoothed and fake-quantized to int8 then restored to BF16; activations
stay BF16. All comparisons against SmoothQuant in this paper are on
weight storage and memory bandwidth, not matmul arithmetic.

\paragraph{QAM-W variants.} QAM-W-polar (polar baseline, $\approx 4.0$~bpw),
QAM-W-4 (joint 2D, $B\!=\!8$, $\approx 4.0$~bpw), and
QAM-W-5.5 (joint 2D, $B\!=\!11$, AWQ-style scaling $\alpha\!=\!0.3$,
$\approx 5.5$~bpw). A sub-4~bpw variant QAM-W-3.5 ($B\!=\!7$,
$\approx 3.5$~bpw) is reported in \cref{sec:low-bit}.

\paragraph{Metrics.} WikiText-2~\citep{merity2017wikitext} stride
perplexity (seq\_len 2048, stride 1024, 16K scored tokens); paired
KL divergence against BF16 (4K scored tokens); six-task
lm-evaluation-harness~\citep{eleuther2024lmevalharness} panel (PIQA \citep{bisk2020piqa},
HellaSwag \citep{zellers2019hellaswag}, COPA \citep{roemmele2011copa}, RTE \citep{wang2018glue}, OpenBookQA \citep{mihaylov2018openbookqa}, LAMBADA-OpenAI \citep{paperno2016lambada}). Full
reproducibility details including model SHAs, calibration setup, and
hardware are in \cref{app:reproducibility}.
\end{shaded}

\subsection{Cross-Model Quality Results}
\label{sec:cross-model}


\begin{figure*}[t]
\centering
\includegraphics[width=\linewidth]{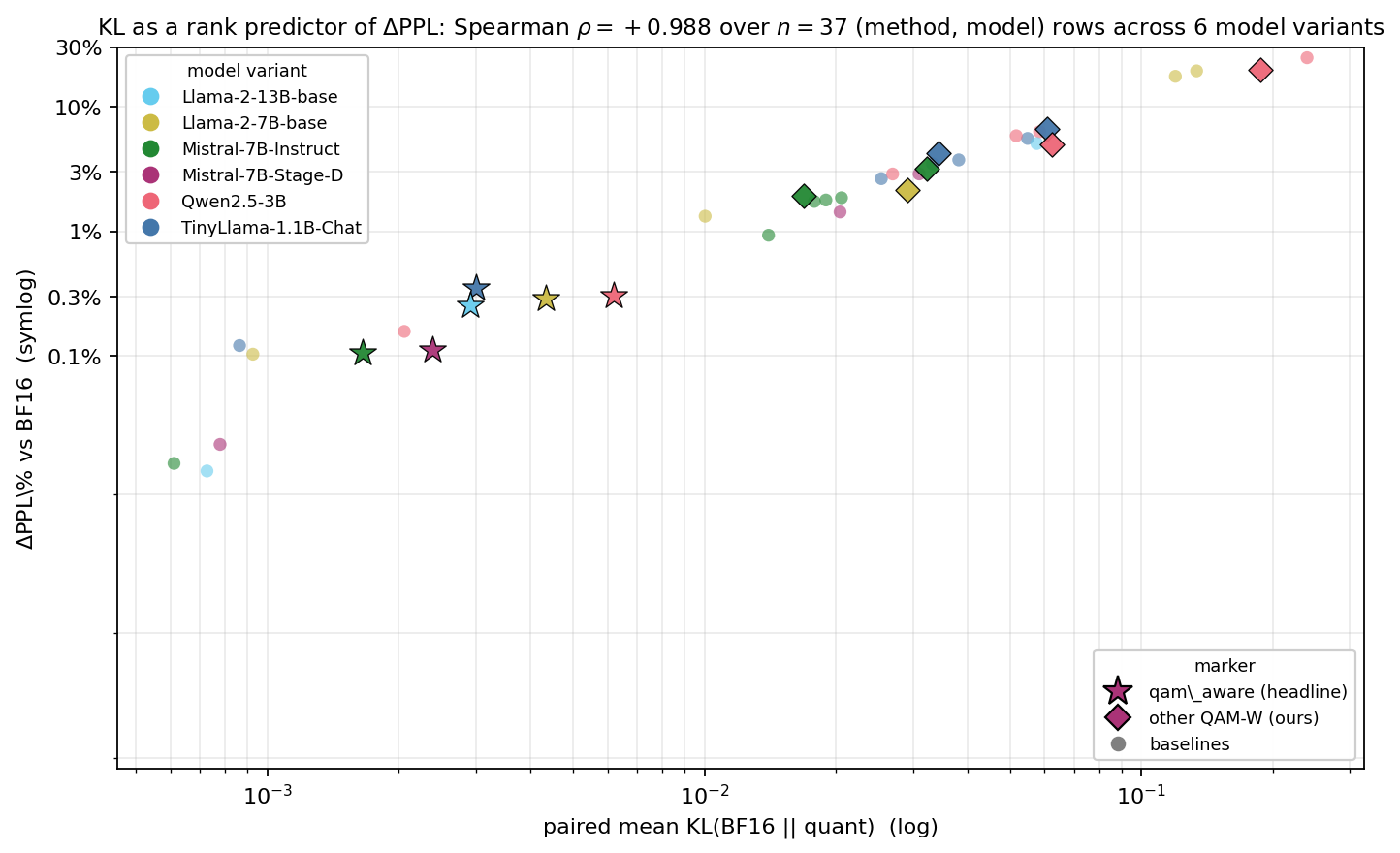}
\caption{Empirical check of \cref{thm:composite}: measured paired KL
(log $x$) vs.\ measured $\Delta$PPL\% (symlog $y$) for every (method,
model) row. Only the headline points and the two large outlier points
are labelled to avoid clutter. The pooled rank correlation
across all 37 (method, model) rows spanning six model variants
is Spearman $\rho = 0.99$.}
\label{fig:kl-vs-dppl}
\end{figure*}

\begin{shaded}
\Cref{tab:cross-model-delta} summarizes the cross-model quality story.
\end{shaded}

\begin{table*}[t]
\centering
\caption{Cross-architecture quantization sensitivity. Each column reports
$\Delta$PPL\% versus that model's BF16 reference (anchor PPL in the column
header); rows are sorted by ascending mean $\Delta$ across models.
QAM-W-5.5 is the lowest-bpw method in the
BF16-quality envelope: the only sub-8 bpw method that stays within
$\pm 0.4\%$ of BF16 on every model.
Detailed per-model breakdowns with paired KL and harness average are in
\cref{app:per-model-tables}.\\
\textit{Reader key:} {\dn} lower is better. Shaded rows ({\starours})
mark QAM-W (ours).}
\label{tab:cross-model-delta}
\setlength{\tabcolsep}{6pt}
\begin{tabular}{lrrrr}
\toprule
method & bpw\dn &
\shortstack[r]{TinyLlama-1.1B\\$\Delta$ vs 7.1499\dn} &
\shortstack[r]{Qwen2.5-3B\\$\Delta$ vs 6.7826\dn} &
\shortstack[r]{Mistral-7B\\$\Delta$ vs 4.8883\dn} \\
\midrule
SmoothQuant W8A8                  & 8.1 & $+0.1\%$ & $+0.2\%$  & $+0.0\%$ \\
\our \starours QAM-W-5.5 & 5.5 & $+0.4\%$ & $+0.3\%$  & $+0.1\%$ \\
GPTQ W4A16 g128                   & 4.1 & $+2.6\%$ & $+2.9\%$  & $+0.9\%$ \\
AutoRound W4A16 g128                & 4.1 & $+3.5\%$ & $+5.8\%$  & $+1.7\%$ \\
AWQ W4A16 g128                    & 4.1 & $+3.7\%$ & $+6.3\%$  & $+1.8\%$ \\
\our \starours QAM-W-4             & 4.0 & $+4.2\%$ & $+5.0\%$  & $+1.9\%$ \\
RTN W4A16 g128                    & 4.1 & $+5.6\%$ & $+24.6\%$ & $+1.9\%$ \\
\our \starours QAM-W-polar          & 4.0 & $+6.6\%$ & $+19.6\%$ & $+3.2\%$ \\
\bottomrule
\end{tabular}
\end{table*}

\begin{shaded}
\paragraph{QAM-W-5.5 reaches the BF16 envelope at the lowest bitrate.}
At $\approx 5.5$~bpw, QAM-W-5.5 stays within
$\pm 0.4\%$ of BF16 perplexity on all three models. SmoothQuant W8A8
reaches the same envelope but at $\approx 8.1$ bpw --- about $47\%$
more weight bits.

\textbf{Calibration narrows the cross-model spread.} Uncalibrated RTN 
costs +1.9\% on Mistral 7B but +24.6\% on Qwen 2.5 3B — a 12$\times$ 
spread driven by architecture alone. Among calibrated methods 
(GPTQ, AWQ, AutoRound, QAM-W-5.5) the cross-model spread is much 
tighter. Single-model comparisons among uncalibrated quantizers can 
therefore mislead in a way that calibrated comparisons typically do not.

\paragraph{Joint 2D coding beats polar at equal bitrate.}
At $\approx 4$ bpw, QAM-W-4 beats QAM-W-polar by 1–3 pp of $\Delta$PPL 
on the quantization-tolerant models (TinyLlama, Mistral) and by 14.6 pp 
on the quantization-sensitive Qwen 2.5. The joint codebook captures 
pair-level correlations that the polar factorization discards; the gain 
is largest precisely where the model is most sensitive to codec choice.
\end{shaded}

\subsubsection{Cross-Model Slope}

\begin{figure*}[t]
\centering
\includegraphics[width=\linewidth]{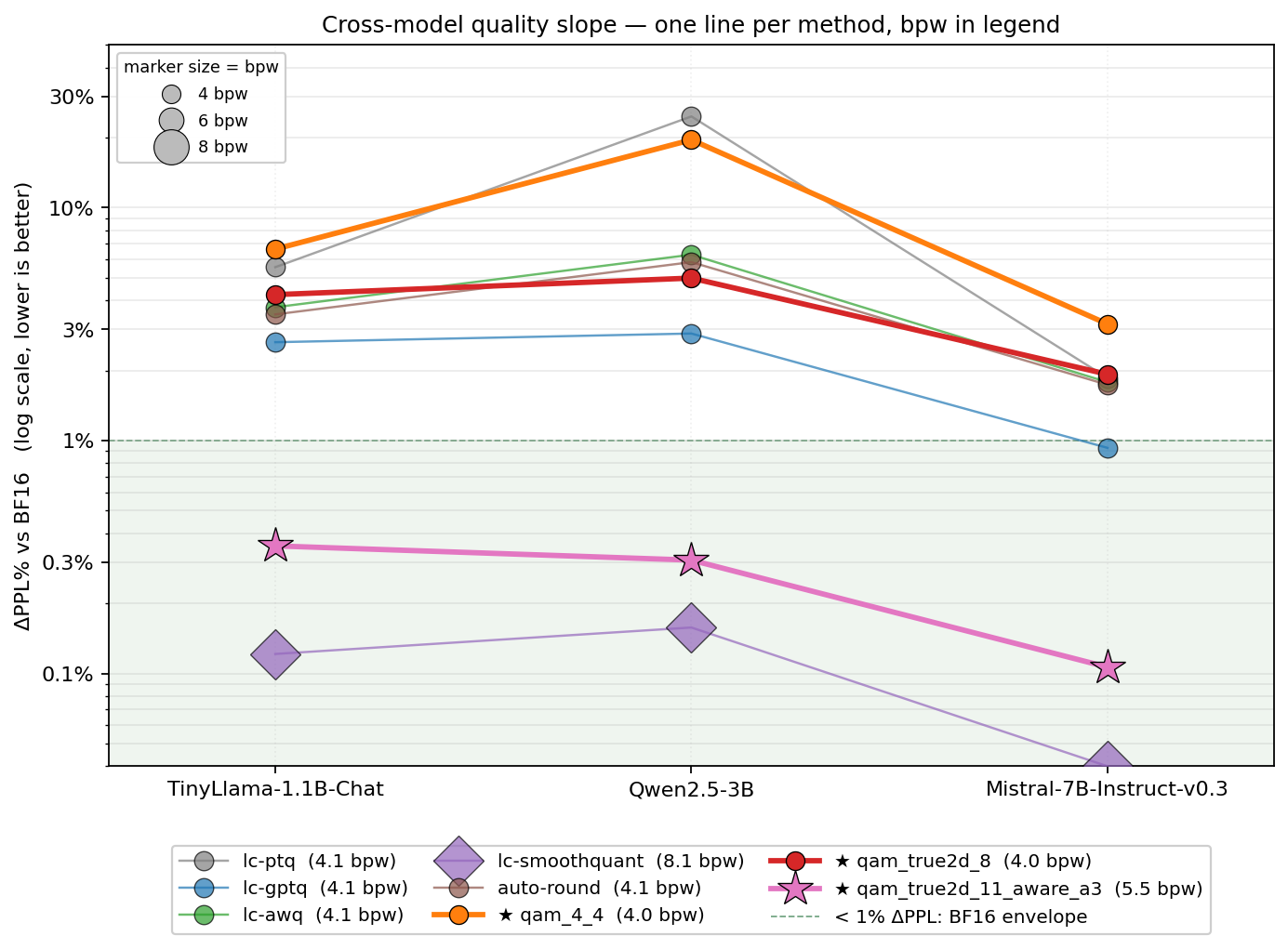}
\caption{Cross-model $\Delta$PPL\% slope. One line per method; the shaded
band at the bottom is the $<\!1\%$ BF16-envelope zone. Two methods are
drawn thicker as the headline: QAM-W-5.5 ($\star$,
$\approx 5.5$ bpw) sits inside the envelope on all three architectures;
SmoothQuant ($\diamondsuit$, $\approx 8.1$ bpw) reaches the same
envelope but at $\sim 47\%$ higher bitrate.}
\label{fig:slope-dppl}
\end{figure*}

\begin{shaded}
Across the $37$ (method, model) rows assembled in this study, paired KL
and $\Delta$PPL\% are rank-correlated at Spearman $\rho = 0.99$
(\cref{fig:kl-vs-dppl}), consistent with the composite
bound in \cref{sec:analysis}. Multi-seed stride-offset rescoring confirms
the ranking is robust (std $0.03$--$0.08$~pp;
\cref{app:multiseed}).
\end{shaded}

\subsection{Iso-bpw Frontier Comparison}
\label{sec:frontier}


\begin{shaded}
To position QAM-W against the rotated-codebook frontier, we run an
iso-bpw comparison on Llama-2-7B-base~\citep{touvron2023llama2}
against QTIP~\citep{tseng2024qtip} (4-bit and 2-bit) and
AQLM~\citep{egiazarian2024aqlm} (2-bit), using published HuggingFace
checkpoints dequantized to BF16 before scoring under our protocol.
\end{shaded}

\begin{table*}[t]
\centering
\caption{Frontier comparison on Llama-2-7B-base, sorted by ascending
bpw. SmoothQuant W8A8 bpw is over MLP weights only; QTIP, AQLM, and
the QAM-W rows quantize the full weight pool per their published
configurations.\\
\textit{Reader key:} {\dn} lower is better, {\up} higher is better.
Shaded rows ({\starours}) are QAM-W (ours).}
\label{tab:frontier-llama2}
\setlength{\tabcolsep}{5pt}
\begin{tabular}{lrrrrr}
\toprule
config & bpw\dn & PPL\dn & $\Delta$PPL\%\dn & mean KL\dn & harness avg\up \\
\midrule
BF16 (ref)                                & 16.00 & 4.7527 & ---       & ---     & 0.6990 \\
SmoothQuant W8A8                          &  8.13 & 4.7576 & $+0.10\%$ & 0.0009  & 0.7009 \\
\our \starours QAM-W-5.5                  &  5.51 & 4.7664 & $+0.29\%$ & 0.0043  & 0.7008 \\
QTIP 4-bit (\texttt{relaxml})             &  4.00 & 4.8155 & $+1.32\%$ & 0.0100  & 0.7027 \\
\our \starours QAM-W-4                    &  4.00 & 4.8551 & $+2.15\%$ & 0.0292  & 0.6993 \\
QTIP 2-bit (\texttt{relaxml})             &  2.00 & 5.5841 & $+17.49\%$ & 0.1195 & 0.6764 \\
AQLM 2-bit 1$\times$16 (\texttt{ISTA-DASLab}) & 2.03 & 5.6716 & $+19.33\%$ & 0.1337 & 0.6782 \\
\bottomrule
\end{tabular}
\end{table*}

\begin{shaded}
\paragraph{QAM-W-5.5 transfers; QTIP leads at 4~bpw.}
QAM-W-5.5 on Llama-2-7B-base reaches $+0.29\%$~$\Delta$PPL, within
the same envelope as the cross-model rows. At iso-4-bpw, QTIP-4Bit
($+1.32\%$, KL~$0.010$) beats QAM-W-4 ($+2.15\%$, KL~$0.029$):
QTIP's $E_8$ lattice trellis and BlockLDLQ rounding outpace QAM-W's
$d\!=\!2$ Lloyd codebook at a strict 4-bit ceiling.
QAM-W-5.5 improves over QTIP-4Bit only at $\sim\!1.4\times$ the bit
budget --- a different operating point, not a Pareto win.

\paragraph{Low-bpw frontier.}
At $\approx 2$~bpw, QTIP and AQLM take comparable PPL hits
($+17$--$19\%$). QAM-W has no 2-bpw operating point; its closest
variant is QAM-W-3.5 at $\approx 3.5$~bpw (\cref{sec:low-bit}).
\end{shaded}

\subsection{Operating Points and Downstream Behavior}


\subsubsection{Low-Bit Operating Point: QAM-W-3.5}
\label{sec:low-bit}

\begin{shaded}
A $B\!=\!7$ codebook ($128$ codewords) drops QAM-W to
$\approx 3.5$~bpw. On TinyLlama, activation-aware scaling
recovers $\approx 0.5$~bits of effective precision:
QAM-W-3.5 ($+4.1\%$~$\Delta$PPL, KL~$0.037$) essentially ties
the unscaled QAM-W-4 at $4.0$~bpw ($+4.2\%$, KL~$0.034$).
Without scaling the 7-bit codec costs $+7.6\%$.
On Qwen~2.5 the 128-entry codec is not yet Pareto-competitive
($+7.5\%$, worse than GPTQ~\citep{frantar2022gptq} at $+2.9\%$),
indicating $\sim\!4$~bpw as the lower bound at which the 2D codec
geometry pays off on sensitive architectures.
On Mistral~7B, 3.5~bpw costs only $+1.98\%$.
Full results in \cref{app:low-bit-details}.
\end{shaded}

\subsubsection{Downstream Task Behavior}
\label{sec:harness}

\begin{shaded}
A six-task lm-evaluation-harness~\citep{eleuther2024lmevalharness}
panel corroborates the PPL story. LAMBADA is the most
quantization-sensitive task; PIQA~\citep{bisk2020piqa} and
COPA~\citep{roemmele2011copa} are nearly flat across methods within
each model; per-task 95\% CI is $\pm 2$--$3$~pp at the evaluation
limits used, so within-model cross-method differences inside that
interval should be read as ties.
Per-task breakdowns and Pareto plots are in
\cref{app:per-model-tables}.

\paragraph{MMLU.}
\label{subsec:mmlu}
QAM-W-5.5 and BF16 scored on the full 57-subject
MMLU~\citep{hendrycks2021mmlu} panel match within $0.2$~pp on all
three models (\cref{tab:mmlu}), confirming the codec's quality envelope
extends beyond the perplexity metric.
\end{shaded}

\begin{table}[t]
\centering
\caption{MMLU macro-average accuracy (57 subjects,
\texttt{limit=300}). QAM-W-5.5 matches BF16 within $0.2$~pp on all
three models. \textit{Reader key:} {\up} higher is better.}
\label{tab:mmlu}
\setlength{\tabcolsep}{4pt}
\begin{tabular}{lrrr}
\toprule
model & BF16\up & QAM-W-5.5\up & $\Delta$ (pp) \\
\midrule
TinyLlama-1.1B       & 0.2489 & 0.2507 & $+0.18$ \\
Qwen2.5-3B           & 0.6700 & 0.6685 & $-0.15$ \\
Mistral-7B            & 0.6150 & 0.6153 & $+0.03$ \\
\bottomrule
\end{tabular}
\end{table}

\section{Conclusion}
\label{sec:conclusion}


\begin{shaded}
The activation-aware joint-2D variant of QAM-W reaches the BF16 quality
envelope at $\approx\!5.5$~bpw on all five models tested, with
$\Delta$PPL~$\le 0.4\%$ on each. Three assumptions underpin the codec:
(i)~weight-row coordinates are correlated and admit a homogenizing
orthogonal rotation; (ii)~the rotated pairs concentrate near a circular
Gaussian suitable for joint 2D Lloyd coding; (iii)~per-channel
activation magnitudes provide a rescaling that aligns codec error with
the low-eigenvalue directions of $M_x$. When any of these weakens ---
as on Qwen~2.5~3B --- all three are needed to recover quality.

In practice, QAM-W-5.5 ($\approx\!5.5$~bpw) is the quality-preserving
choice, matching SmoothQuant~\citep{xiao2024smoothquant} W8A8 at
$\approx\!32\%$ fewer weight bits. QAM-W-4 ($\approx\!4$~bpw) serves
as the iso-bitrate comparison against W4A16 baselines and directly demonstrates that joint 2D coding outperforms polar by 1–3 pp on quantization-tolerant models and by over 14 pp on the most quantization-sensitive one. QAM-W-3.5 ($\approx\!3.5$~bpw) is viable on quantization-tolerant architectures but not yet Pareto-competitive on sensitive ones.

More broadly, among uncalibrated quantizers, sensitivity is dominated 
by the model architecture: RTN costs +1.9\% on Mistral but +24.6\% on Qwen, a 12$\times$ difference from changing the model alone. Among calibrated methods, the spread is much smaller. Single-model comparisons can therefore mislead.
Natural follow-ups include a learned 2D codebook fit to each
checkpoint's empirical pair distribution and a fused
dequantization-matmul kernel to translate the bitrate saving into
inference latency gains.
\end{shaded}

\section*{Limitations}
\label{sec:limitations}

\paragraph{Scope of quantized parameters.} The cross-model headline tables
(\cref{sec:cross-model}) report bpw over MLP gate/up/down projections only;
attention weights and the LM head remain BF16. The MLP + attention
extension in \cref{sec:stage-d} adds the four self-attention
projections ($q$, $k$, $v$, $o$) for the BF16 reference and the four
strongest sub-8\,bpw quantized methods on Mistral 7B and shows the
ranking is preserved, but the cross-model tables themselves do not yet
include an MLP + attention row for TinyLlama and Qwen; the LM head
remains untouched in every reported configuration.

\paragraph{Codebook training cost.} Lloyd-Max codebook training for the joint
2D codec is a one-time cost of $\sim 5$~s per bitrate level. This is
negligible compared to inference, but it is non-zero and is performed on a
synthetic unit circular Gaussian rather than per-checkpoint --- a learned
codebook conditioned on the empirical post-rotation pair distribution might
buy additional precision.

\paragraph{Activation-aware $\alpha$ is grid-tuned, not learned.} The AWQ-style
per-channel scaling exponent $\alpha=0.3$ was chosen from a small grid
$\{0.0, 0.3, 0.5, 0.8\}$ on TinyLlama and held fixed across models. A
per-model or per-layer $\alpha$ schedule would likely improve quality on
Qwen 2.5, where the 7-bit codec is most $\alpha$-sensitive.

\paragraph{Architecture coverage.} The cross-model tables rest on three
instruction-tuned models from three families (TinyLlama, Qwen 2.5,
Mistral) plus two Llama-2 base models at 7B and 13B
(\cref{sec:frontier,sec:scale-13b}) --- five models from four families
across $1.1$B--$13$B parameters. Beyond this range, $70$B and
contemporary families (Llama 3, Gemma 2, DeepSeek) would further tighten
the architecture-sensitivity claim, particularly to test whether the
$\approx\!5.5$ bpw envelope continues to hold at the $70$B parameter
scale where activation statistics diverge further from the smaller-model
distributions calibrated against here. The choice to evaluate on these
architecturally conventional models is deliberate: all five are dense
decoder-only transformers with standard multi-head or grouped-query
attention, for which mature weight-loading and evaluation libraries make
the codec comparison free of model-loading confounds. Contemporary
(2026-era) architectures are not yet reliably supported end-to-end by
stock libraries --- e.g.\ Gemma-4 combines per-layer input embeddings
with interleaved local/global sliding-window attention and must be loaded
through a forked model implementation. Extending QAM-W to such
architectures, via a dedicated inference engine, is left to future work.

\paragraph{Iso-4-bpw quality trails the QTIP frontier.} \Cref{sec:frontier}
reports QTIP-4Bit ahead of QAM-W's joint-2D codec at iso-4-bpw on
Llama-2-7B-base; the joint-2D Lloyd codebook trades some
rate-distortion ceiling for training cheapness and decode simplicity
(a single $2^B$-entry table lookup rather than a lattice
nearest-point search). Closing the iso-4-bpw gap is the most natural
extension: candidate paths are a richer codebook training distribution
per checkpoint (\cref{sec:conclusion}) and tighter Lloyd optimization
with QTIP-style BlockLDLQ rounding.

\paragraph{Harness precision.} The six-task harness panel uses
\texttt{limit=300} (Qwen, Mistral) and \texttt{limit=500} (TinyLlama). The
resulting per-task accuracy noise floor is on the order of one to two
percentage points; cross-method differences smaller than that should be read
as ties.

\paragraph{No fused dequantization-matmul kernel.} The reported bpw is a
storage metric. The current pipeline dequantizes weights to BF16 before the
linear-layer matmul, so the bitrate saving shows up in weight storage and
host-to-device memory bandwidth, not in matmul arithmetic. The bpw numbers
should be read as memory-footprint claims until a fused kernel is available.
Building this kernel --- and with it support for architectures beyond the
conventional dense transformers studied here --- is the primary engineering
follow-up.

\paragraph{Comparison scope.} The evaluation covers five models from four
families and seven baselines under one unified protocol. Notably absent are
iso-bpw comparisons against learned-rotation methods (QuaRot, SpinQuant)
and a direct $d\!=\!1$ vs.\ $d\!=\!2$ ablation against
PolarQuant~\citep{vicentino2026polarquant}, which would isolate the
codebook-dimension contribution from the shared rotation-plus-scaling
pipeline. Scale beyond 13B and broader downstream benchmarks (MMLU-Pro,
BBH, IFEval, GSM8K) are also not covered.

% ...

\bibliography{references}

\appendix
\section{Codec-Level Analysis}
\label{sec:codec-analysis}

With the encoder/decoder fixed (\cref{sec:method}), the two codec operations
that determine directional reconstruction error are the block-Hadamard rotation
and the per-pair quantizer. Each is analyzed in turn: the rotation as an exact
isometry, and the polar amplitude-phase decomposition as the baseline against
which the joint 2D codebook is judged. Layer-output effects --- how this
directional error translates to perturbations under real activations --- are
deferred to \cref{sec:activation}.

\subsection{Analysis I: Rotation Isometry}

The normalized row is transformed by a sign-masked block-Hadamard rotation.
This rotation is an exact isometry in real arithmetic, so it cannot add
mathematical distortion by itself. Its practical role is to reduce coordinate
coherence and spread energy within each block before quantization; the
circular-Gaussian model used by the pair codebooks is calibrated empirically and
does not follow from orthogonality alone.

\begin{definition}[Unnormalized Walsh-Hadamard matrix]
For any integer $b = 2^m$ with $m \ge 0$, the unnormalized Walsh-Hadamard
matrix $H_b \in \{-1,+1\}^{b \times b}$ is defined recursively by $H_1=[1]$
and
\begin{equation}
  H_{2b} =
  \begin{bmatrix}
    H_b & H_b \\
    H_b & -H_b
  \end{bmatrix}.
\end{equation}
It satisfies $H_b^\top H_b = bI$ and $H_b=H_b^\top$.
\end{definition}

\begin{definition}[Block-normalized Hadamard transform]
Let $b$ be a power of two dividing $d_{\text{in}}$. The normalized block-Hadamard matrix
$F \in \mathbb{R}^{d_{\text{in}} \times d_{\text{in}}}$ is block diagonal with $d_{\text{in}}/b$ identical blocks,
each equal to $(1/\sqrt{b})H_b$:
\begin{equation}
  F =
  \frac{1}{\sqrt{b}}\,
  \underbrace{
    H_b \oplus \cdots \oplus H_b
  }_{d/b \text{ blocks}} .
\end{equation}
\end{definition}

\begin{definition}[Sign mask]
Let $s_1, \dots, s_{d_{\text{in}}} \in \{\pm 1\}$ be a deterministic sequence generated from a fixed 64-bit seed. The sign mask is the diagonal matrix
$S = \mathrm{diag}(s_1, \dots, s_{d_{\text{in}}})$.
\end{definition}

\begin{lemma}[Block-Hadamard rotation isometry]
\label{lem:rotation}
Assume $d_{\text{in}} = mb$ for integers $m \geq 1$ and $b = 2^q$ with $q \geq 1$.
Let $F = I_m \otimes b^{-1/2} H_b$ be the block-normalized Hadamard
matrix and $S = \mathrm{diag}(s_1, \dots, s_{d_{\text{in}}})$ with 
$s_j \in \{\pm 1\}$.
With $R_{\mathrm{fwd}} = FS$ and $R_{\mathrm{inv}} = SF$, both maps
are orthogonal in exact real arithmetic
($R_{\mathrm{fwd}}^\top R_{\mathrm{fwd}} = R_{\mathrm{inv}}^\top R_{\mathrm{inv}} = I$),
mutually inverse ($R_{\mathrm{inv}}R_{\mathrm{fwd}} = I$), and
preserve norms, distances, inner products, and per-block energy.
\end{lemma}

\begin{proof}
$H_b^\top H_b = bI$ and $H_b=H_b^\top$, so each normalized block
$b^{-1/2}H_b$ is orthogonal and symmetric. $F$ is block diagonal with
these blocks, giving $F^\top F = I$ and $F^\top = F$. $S$ is diagonal
with $\pm 1$ entries, so $S^\top = S$ and $S^2 = I$. Then
$R_{\mathrm{fwd}}^\top R_{\mathrm{fwd}} = S F^\top F S = I$ and
$R_{\mathrm{inv}} R_{\mathrm{fwd}} = SF \cdot FS = S F^2 S = I$. Norm,
distance, and inner-product preservation follow from orthogonality;
per-block energy preservation follows because $F$ and $S$ act
independently within each block of the partition.
\end{proof}

\begin{remark}
The lemma is purely geometric. It proves that the rotation contributes
no mathematical distortion and that quantization error has the same
Euclidean norm before and after the inverse rotation. It does not prove that the rotated
coordinates are independent or Gaussian. If the signs are idealized as
independent Rademacher variables, then for a fixed vector $v$ and coordinate
$j$ in block $B(j)$,
\[
\begin{gathered}
  \mathbb{E}_S[(R_{\mathrm{fwd}}v)_j]=0, \\
  \mathbb{E}_S[(R_{\mathrm{fwd}}v)_j^2]=\frac{\|v_{B(j)}\|_2^2}{b}.
\end{gathered}
\]
This explains the energy-spreading intuition within each block. Approximate
Gaussianity requires the usual incoherence condition that no small number of
input coordinates dominates the signed sum, and is checked empirically through
calibration.
\end{remark}

\begin{remark}
The forward and inverse matrices are generally different because $F$ and $S$ do
not commute. The sign-masked transform is orthogonal, but it is not generally
symmetric and is not generally an involution.
\end{remark}

\subsection{Analysis II: Polar Pair Distortion}

The second codec-level theorem analyzes the polar baseline from
\cref{sec:method}. The theorem gives a deterministic amplitude-phase
decomposition; the following corollary adds the Rayleigh/Lloyd source model
used to compare independent amplitude-phase coding against joint 2D QAM.

\begin{theorem}[Polar QAM pair distortion decomposition]
\label{thm:pairwise}
For a complex pair $z=ae^{i\theta}$ and reconstruction
$\hat{z}=\hat{a}e^{i\hat{\theta}}$,
\begin{equation}
  \|z-\hat{z}\|^2
  = (a-\hat{a})^2
    + 2a\hat{a}\bigl(1-\cos(\theta-\hat{\theta})\bigr).
\end{equation}
If $a,\hat{a},\theta,\hat{\theta}$ are random variables with finite second
moments and $\delta=\operatorname{wrap}(\theta-\hat{\theta})$, then
\begin{equation}
  \mathbb{E}\|z-\hat{z}\|^2
  =
  \mathbb{E}(a-\hat{a})^2
  + 2\mathbb{E}\left[a\hat{a}\bigl(1-\cos\delta\bigr)\right].
\end{equation}
\end{theorem}

\begin{proof}
The deterministic identity follows by expanding
\begin{align*}
  |ae^{i\theta}-\hat{a}e^{i\hat{\theta}}|^2
  &= a^2+\hat{a}^2-2a\hat{a}\cos(\theta-\hat{\theta}) \\
  &= (a-\hat{a})^2 \\
  &\quad +2a\hat{a}\bigl(1-\cos(\theta-\hat{\theta})\bigr).
\end{align*}
Because cosine is $2\pi$-periodic, replacing $\theta-\hat{\theta}$ by the
wrapped error $\delta$ does not change the expression. Taking expectations
gives the second claim.
\end{proof}

\begin{corollary}[Rayleigh-Lloyd polar distortion bound]
\label{cor:rayleigh-polar}
Assume the circular-Gaussian model for pair $k$: $a=\sigma_k U$ with
$U\sim\operatorname{Rayleigh}(1)$, $\theta$ is uniform on $(-\pi,\pi]$, and
$a$ and $\theta$ are independent. Let $Q_a$ be an $N_a$-level Lloyd-Max
quantizer for $U$ under squared error, with normalized distortion
$C_{\mathrm{LM}}(N_a)=\mathbb{E}(U-Q_a(U))^2$, and set
$\hat{a}=\sigma_k Q_a(U)$. Let the phase quantizer round $\theta$ to the
nearest one of $N_p$ uniform phase levels, with bin centers aligned to the
grid $\{(2k+1)h : k=0,1,\ldots,N_p-1\}$ modulo $2\pi$, where $h=\pi/N_p$
(so that adjacent phase bin boundaries are spaced by $2h$ and bin centers
sit at midpoints). Then
\begin{equation}
\begin{aligned}
  \mathbb{E}\|z-\hat{z}\|^2
  ={}& \sigma_k^2 C_{\mathrm{LM}}(N_a) \\
  &+ 2\sigma_k^2 M_a(N_a)\eta(N_p),
\end{aligned}
\end{equation}
where
\begin{equation}
\begin{gathered}
  M_a(N_a)=\mathbb{E}\bigl[UQ_a(U)\bigr], \\
  \eta(N_p)=1-\frac{\sin(\pi/N_p)}{\pi/N_p}.
\end{gathered}
\end{equation}
For a centroid Lloyd-Max quantizer, $M_a(N_a)=2-C_{\mathrm{LM}}(N_a)$, hence
\begin{equation}
\begin{aligned}
  \mathbb{E}\|z-\hat{z}\|^2
  = \sigma_k^2\Bigl[
    & C_{\mathrm{LM}}(N_a) \\
    &+ 2\bigl(2-C_{\mathrm{LM}}(N_a)\bigr)\eta(N_p)
  \Bigr]
\end{aligned}
\end{equation}
and
\begin{equation}
\begin{aligned}
  \mathbb{E}\|z-\hat{z}\|^2
  \le \sigma_k^2\Bigl[
    & C_{\mathrm{LM}}(N_a) \\
    &+ \bigl(2-C_{\mathrm{LM}}(N_a)\bigr)\frac{\pi^2}{3N_p^2}
  \Bigr].
\end{aligned}
\end{equation}
\end{corollary}

\begin{proof}
The amplitude contribution is
$\mathbb{E}(a-\hat{a})^2=\sigma_k^2 C_{\mathrm{LM}}(N_a)$ by the definition
of the normalized Lloyd-Max distortion.
Under the circular-Gaussian model, amplitude and phase are independent.
Because $\theta$ is uniform on $(-\pi,\pi]$ and the phase quantizer is uniform
with bin centers aligned to multiples of $2h$ (the hypothesis on the phase
grid), nearest-level rounding produces a wrapped phase error $\delta$ that is
uniform on $[-h,h]$ with $h=\pi/N_p$, independent of $a$ and $\hat{a}$. Hence
\begin{equation}
\begin{aligned}
  \mathbb{E}(1-\cos\delta)
  &= 1-\frac{1}{2h}\int_{-h}^{h}\cos t\,dt \\
  &= 1-\frac{\sin h}{h}
   = \eta(N_p).
\end{aligned}
\end{equation}
The phase contribution is therefore
$2\mathbb{E}[a\hat{a}]\eta(N_p)
=2\sigma_k^2M_a(N_a)\eta(N_p)$.

For a centroid Lloyd-Max quantizer, $Q_a(U)=\mathbb{E}[U\mid Q_a(U)]$ on each
cell. By the tower property and the centroid identity, and noting that
$Q_a(U)$ is $\sigma(Q_a(U))$-measurable so it pulls inside the conditional
expectation,
\begin{equation}
\begin{aligned}
  M_a(N_a) &=\mathbb{E}[UQ_a(U)] \\
  &=\mathbb{E}\bigl[\mathbb{E}[U\mid Q_a(U)]\cdot Q_a(U)\bigr] \\
  &=\mathbb{E}[Q_a(U)^2].
\end{aligned}
\end{equation}
Since $\mathbb{E}U^2=2$ for the unit Rayleigh law,
\begin{equation}
\begin{aligned}
  C_{\mathrm{LM}}(N_a)
  &= \mathbb{E}(U-Q_a(U))^2 \\
  &= \mathbb{E}U^2 - \mathbb{E}[Q_a(U)^2]
   = 2-M_a(N_a),
\end{aligned}
\end{equation}
which gives $M_a(N_a)=2-C_{\mathrm{LM}}(N_a)$. Finally, the bound
$1-\sin h/h\le h^2/6$ follows from the Taylor expansion
$\sin h=h-h^3/6+O(h^5)$, valid for $h\in[0,\pi]$ (the remainder is
non-negative on this interval since the next term is $+h^5/120$ and the
series is alternating with decreasing magnitude). Substituting
$h=\pi/N_p$ gives
$2(2-C_{\mathrm{LM}}(N_a))\eta(N_p)
\le(2-C_{\mathrm{LM}}(N_a))\pi^2/(3N_p^2)$.
\end{proof}

\begin{remark}
The bound is an idealized source-model statement. The implementation uses a
finite Rayleigh integration cutoff $R_{\max}=8$ and f32 codebook levels, so its
$C_{\mathrm{LM}}$ is the numerically integrated distortion of that practical
codebook. The Rayleigh tail beyond $8\sigma_k$ is negligible for the reported
bit budgets, but the exact equality above should be read as the population
calculation for the ideal Rayleigh source and nearest uniform phase quantizer.
\end{remark}

\subsection{Rotated-Domain Error Accounting}

Each complex pair contains two scalar coordinates. Because the inverse rotation
is an isometry, the squared error of the unnormalized inverse-rotated direction
is the sum of pairwise errors:
\begin{equation}
\begin{gathered}
  \|u-\tilde{u}\|_2^2
  = \|R_{\mathrm{fwd}}u-\hat{y}\|_2^2
  = \sum_k \|z_k-\hat{z}_k\|^2, \\
  \tilde{u}=R_{\mathrm{inv}}\hat{y}.
\end{gathered}
\end{equation}
If the decoder subsequently normalizes $\tilde{u}$ before multiplying by the
stored row norm, that final radial projection is a separate step; the equality
above applies to the linear inverse-rotation stage.

\begin{corollary}[Direction and per-coordinate reporting]
\label{cor:direction-reporting}
Let each pair error be
$e_k=\|z_k-\hat{z}_k\|^2$. Before the optional final direction
renormalization,
\begin{equation}
\begin{gathered}
  \|u-\tilde{u}\|_2^2 = \sum_k e_k, \\
  \operatorname{MSE}_{\mathrm{coord}}
  = \frac{1}{d}\sum_k e_k
  = \frac{1}{2}\operatorname{mean}_k(e_k).
\end{gathered}
\end{equation}
If $u$ and $\hat{u}$ are both unit vectors, then the cosine loss is exactly
\begin{equation}
  1-\langle u,\hat{u}\rangle
  =
  \frac{1}{2}\|u-\hat{u}\|_2^2 .
\end{equation}
\end{corollary}

\begin{proof}
The first identity is the rotated-domain error equality above, and the
per-coordinate statement only divides by the $d=2\,\#\{k\}$ scalar coordinates.
For unit vectors,
$\|u-\hat{u}\|_2^2=\|u\|_2^2+\|\hat{u}\|_2^2-2\langle u,\hat{u}\rangle
=2(1-\langle u,\hat{u}\rangle)$.
\end{proof}

\subsection{Empirical Check of the Circular-Gaussian Model}
\label{subsec:qq-check}

\Cref{cor:rayleigh-polar} and the Lloyd codebook training in
\cref{subsec:complex} both invoke the circular-Gaussian model for the
post-rotation pair distribution. Whether this model is a useful idealization
of empirical transformer weights is a separable question that can be
checked directly. For each of the three target architectures, one
representative MLP \texttt{down\_proj} matrix is processed with the
same pipeline used by the codec: a sign-masked block-Hadamard rotation
at $b=1024$ block size, per-row unit normalization, and sampling of
$50{,}000$ rotated complex pairs. The marginal distributions of the
magnitude $|z|$ are then compared to a fitted Rayleigh law, and of the
real part $\Re(z)$ to a fitted zero-mean Gaussian, via QQ plots
(\cref{fig:qq-rotated-pairs}).

\begin{figure*}[t]
\centering
\includegraphics[width=\linewidth]{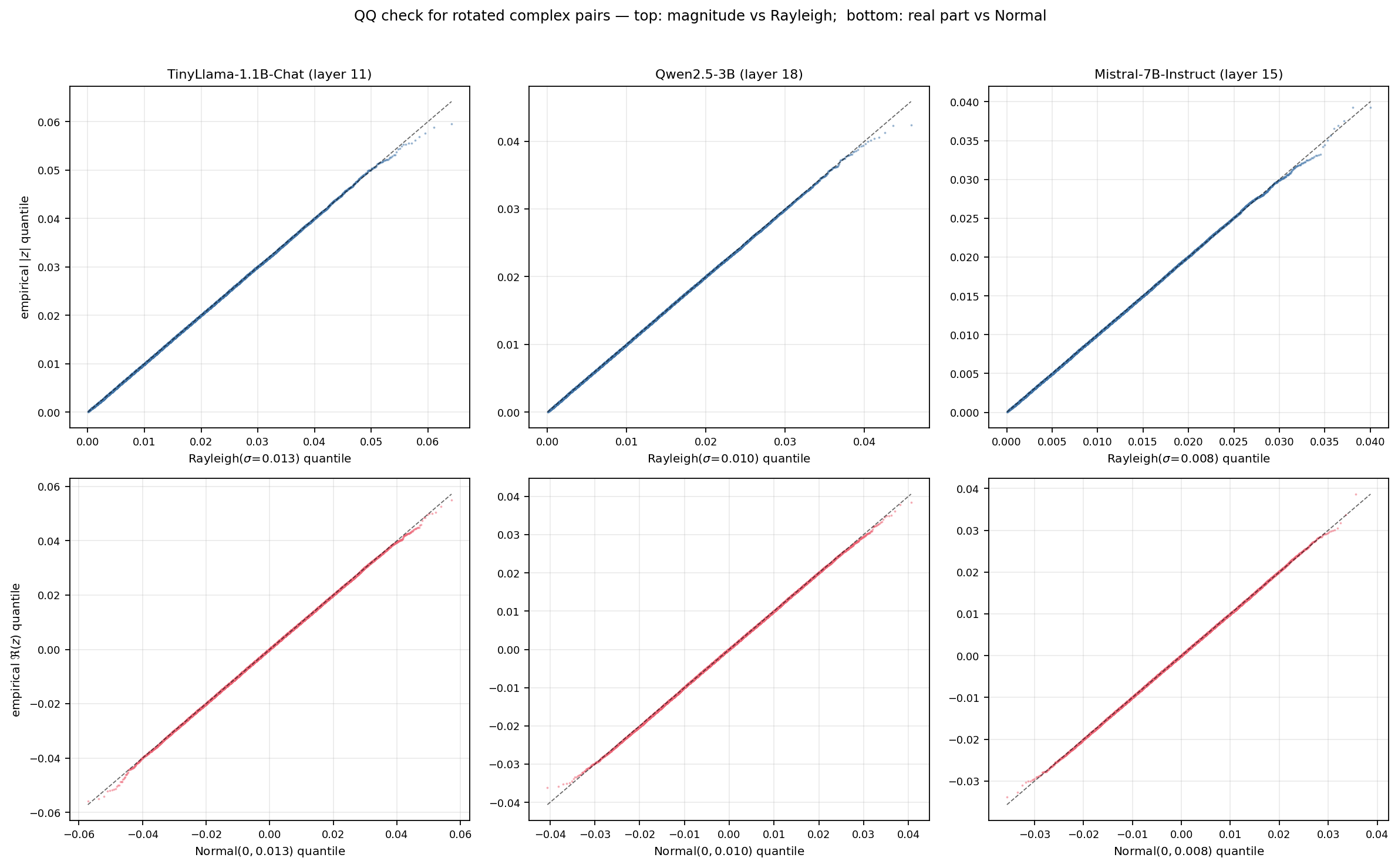}
\caption{QQ plots for the rotated complex-pair marginals of a representative
\texttt{mlp.down\_proj} matrix on each of the three target architectures.
Top row, per architecture: empirical $|z|$ quantiles against a Rayleigh law
fitted on the same sample. Bottom row: empirical $\Re(z)$ quantiles against
a zero-mean Gaussian fitted on the same sample. Sample size per panel is
$50{,}000$ rotated pairs collected from one representative MLP \texttt{down\_proj}
after a sign-masked block-Hadamard rotation at the same $b{=}1024$ block size
used by the codec. The traces lie on the $y{=}x$ identity through the body of
the distribution and depart only in the far tail. The circular-Gaussian
model used by \cref{cor:rayleigh-polar} is therefore a reasonable working
idealization in the regime that dominates the Lloyd objective.}
\label{fig:qq-rotated-pairs}
\end{figure*}

\Cref{fig:qq-rotated-pairs} shows the empirical check: one representative
\texttt{mlp.down\_proj} matrix per architecture, $50{,}000$ rotated complex
pairs after the same sign-masked block-Hadamard rotation the codec uses.
The body of the distribution tracks the Rayleigh (magnitude) and zero-mean
Gaussian (real-part) marginals closely; departures appear only in the far
tail, which the Lloyd training implicitly down-weights through the
density-weighted MSE objective. The post-rotation distribution is therefore
close enough to the circular-Gaussian model that the analytical results
based on it are an informative guide to codec behaviour, not a placeholder
assumption.

\section{Layer Output Error and Activation Awareness}
\label{sec:activation}

\cref{sec:codec-analysis} established that the codec's reconstruction
error is directional and controlled, to first order, by the per-pair
Lloyd distortion. Quality at the
model level, however, is not measured in weight space directly: a quantized
layer is judged by how its output deviates from the BF16 output under the
activations it actually sees. This section bridges weight error to
layer-output error via a closed-form trace identity, then uses that identity
to motivate AWQ-style per-channel scaling.

\paragraph{Notation.} Throughout this and the following sections,
$\hat M_x = X^\top X$ is the unnormalized activation second-moment matrix
($X\in\mathbb{R}^{n\times d_{\mathrm{in}}}$, with $n$ activation rows in
the calibration batch), and $M_x = \hat M_x / n$ is its per-token-normalized
form. Both forms appear; $\hat M_x$ in finite-sample empirical identities
and $M_x$ when stating per-sample expectations or when referring to a
random activation row's second moment.

\subsection{Activation-Weighted Layer Error}

\begin{proposition}[Layer output error identity]
\label{prop:layer}
Let $W,\hat{W}\in\mathbb{R}^{d_{\mathrm{out}}\times d_{\mathrm{in}}}$ with
$\Delta W=W-\hat{W}$, and let
$X\in\mathbb{R}^{n\times d_{\mathrm{in}}}$ be a row-major matrix of layer-input
activation vectors with $n\ge1$. The identity conditions on the same input
activations $X$ being supplied to the original and quantized layer; if earlier
quantized layers perturb $X$, that propagation effect is outside this
single-layer statement. Define the unnormalized activation second-moment matrix
$\hat{M}_x=X^\top X$. Then
\begin{equation}
\begin{aligned}
  \|XW^\top-X\hat{W}^\top\|_F^2
  &= \|X\Delta W^\top\|_F^2 \\
  &= \operatorname{Tr}(\Delta W \hat{M}_x \Delta W^\top).
\end{aligned}
\end{equation}
With $M_x=\frac{1}{n}X^\top X$, the mean per-sample squared output error is
\begin{equation}
  \frac{1}{n}\|X\Delta W^\top\|_F^2
  =
  \operatorname{Tr}(\Delta W M_x \Delta W^\top).
\end{equation}
Equivalently, for a random activation row $x\in\mathbb{R}^{d_{\mathrm{in}}}$
with finite second moment $M_x=\mathbb{E}[x^\top x]$,
\begin{equation}
  \mathbb{E}\left[\|xW^\top-x\hat{W}^\top\|_2^2\right]
  =
  \operatorname{Tr}(\Delta W M_x \Delta W^\top).
\end{equation}
\end{proposition}

\begin{proof}
\begin{align*}
  \|X\Delta W^\top\|_F^2
  &= \operatorname{Tr}\left[(X\Delta W^\top)(X\Delta W^\top)^\top\right] \\
  &= \operatorname{Tr}(X\Delta W^\top\Delta W X^\top) \\
  &= \operatorname{Tr}(\Delta W X^\top X \Delta W^\top).
\end{align*}
Dividing by $n$ gives the empirical mean. The random-activation statement is
the same identity with $M_x=\mathbb{E}[x^\top x]$.
\end{proof}

This identity shows that the relevant layer objective is not unweighted weight
MSE, $\|\Delta W\|_F^2$, but activation-weighted error on the calibration
activations. Unweighted MSE is recovered only when $M_x$ is proportional to the
identity. The matrix $M_x$ is an uncentered second moment; it equals the
covariance only when activations have zero mean. The implementation reports the
relative layer-output RMSE
\begin{equation}
  \frac{\|XW^\top-X\hat{W}^\top\|_F}{\|XW^\top\|_F},
\end{equation}
which is the proposition's numerator normalized by the BF16 layer-output norm.

\begin{corollary}[Diagonal activation-RMS surrogate]
\label{cor:diagonal-rms}
If the normalized activation second moment is diagonal,
$M_x=\operatorname{diag}(r_1^2,\ldots,r_{d_{\mathrm{in}}}^2)$, then
\begin{equation}
  \frac{1}{n}\|X\Delta W^\top\|_F^2
  =
  \sum_{i=1}^{d_{\mathrm{out}}}
  \sum_{j=1}^{d_{\mathrm{in}}}
  r_j^2(\Delta W_{ij})^2 .
\end{equation}
For empirical calibration activations,
$r_j^2=n^{-1}\sum_{t=1}^n X_{tj}^2$ is the squared RMS of input channel $j$.
\end{corollary}

\begin{proof}
Substitute the diagonal $M_x$ into \cref{prop:layer}:
\[
\begin{aligned}
  \operatorname{Tr}(\Delta W M_x\Delta W^\top)
  &= \sum_i (\Delta W M_x\Delta W^\top)_{ii} \\
  &= \sum_{i,j} r_j^2(\Delta W_{ij})^2 .
\end{aligned}
\]
The empirical formula for $r_j^2$ is the $j$th diagonal entry of
$n^{-1}X^\top X$.
\end{proof}

When $M_x$ is not diagonal, this expression is a diagonal approximation that
ignores cross-channel activation correlations in the current input basis. It
still gives the operative AWQ-style intuition: errors in high-RMS channels
matter more to the layer output than equal-sized errors in low-RMS channels.

\subsection{Activation-Aware Scaling}

QAM-W uses AWQ-inspired per-input-channel scaling before quantization. For scaling
factors $s_j>0$ and $D_s=\operatorname{diag}(s_1,\ldots,s_{d_{\mathrm{in}}})$,
\begin{align}
  W_{\mathrm{scaled}} &= W D_s, \\
  \hat{W}_{\mathrm{scaled}} &= \operatorname{Quantize}(W_{\mathrm{scaled}}), \\
  \hat{W} &= \hat{W}_{\mathrm{scaled}}D_s^{-1}.
\end{align}
The implementation sets $s_j\propto r_j^\alpha$ with
$\alpha\in\{0.3,0.5,0.8\}$, normalizes the geometric mean of the scales to one,
and clamps scales to $[1/16,16]$.

\begin{proposition}[Sufficient-condition model for activation-aware scaling]
\label{prop:aware}
Let $E_s=W_{\mathrm{scaled}}-\hat{W}_{\mathrm{scaled}}$ be the quantization
error in the scaled domain, with expectation taken over the quantizer
randomness alone for a fixed (deterministic) weight matrix $W$. Assume:
\begin{itemize}
  \item[(A1)] \emph{Scale-invariant distortion:} the per-coordinate
              quantization error satisfies
              $\mathbb{E}[(E_s)_{ij}^2]\le c_{ij}$ for bounded constants
              $c_{ij}$ that do not depend on the scaling factors
              $s_1,\ldots,s_{d_{\mathrm{in}}}$.
  \item[(A2)] \emph{Diagonal activation second moment:}
              $M_x=\operatorname{diag}(r_1^2,\ldots,r_{d_{\mathrm{in}}}^2)$,
              as in \cref{cor:diagonal-rms}.
  \item[(A3)] \emph{Unclamped power-law scaling:} $s_j\propto r_j^\alpha$
              for some $\alpha\in[0,1]$, with the global geometric-mean
              normalization absorbed into $W$ and no per-channel clipping.
\end{itemize}
Then the de-scaled error $\Delta W^{(s)}=W-\hat{W}=E_sD_s^{-1}$ satisfies
$\mathbb{E}[(\Delta W_{ij}^{(s)})^2]
=\mathbb{E}[(E_s)_{ij}^2]/s_j^2\le c_{ij}/s_j^2$,
where the equality uses $\Delta W_{ij}^{(s)}=(E_s)_{ij}/s_j$ with $s_j$
deterministic and the inequality is (A1). Under (A2) and
\cref{cor:diagonal-rms}, the activation-weighted error satisfies
\begin{equation}
  \mathcal{E}_{\mathrm{act}}(s)
  =
  \sum_{i,j} r_j^2\,
  \mathbb{E}\bigl[(\Delta W_{ij}^{(s)})^2\bigr]
  \le
  \sum_{i,j} c_{ij}\frac{r_j^2}{s_j^2}.
\end{equation}
Under (A3) the channel factor changes from $r_j^2$ to $r_j^{2(1-\alpha)}$,
attenuating the activation-weighted contribution of high-RMS channels.
\end{proposition}

\begin{proof}
With $W$ fixed and the expectation taken over the quantizer randomness only,
$\Delta W_{ij}^{(s)}=(E_s)_{ij}/s_j$ with $s_j$ deterministic, so
$\mathbb{E}[(\Delta W_{ij}^{(s)})^2]=\mathbb{E}[(E_s)_{ij}^2]/s_j^2$.
Assumption (A1)---independence of $c_{ij}$ from $s$---is what makes this
division valid as a uniform bound across the choice of scales:
$\mathbb{E}[(\Delta W_{ij}^{(s)})^2]\le c_{ij}/s_j^2$. Substituting into the
diagonal-$M_x$ identity from \cref{cor:diagonal-rms} (which is (A2)) gives
$\mathcal{E}_{\mathrm{act}}(s)
=\sum_{i,j}r_j^2\mathbb{E}[(\Delta W_{ij}^{(s)})^2]
\le\sum_{i,j}c_{ij}r_j^2/s_j^2$.
Under (A3), $s_j\propto r_j^\alpha$ with global normalization absorbed into
$W$ and no clipping, so the channel factor $r_j^2/s_j^2$ is proportional to
$r_j^{2(1-\alpha)}$.
\end{proof}

The proposition is a design model rather than a universal guarantee; it
describes how (A1)--(A3) interact, and the deployed quantizer satisfies all
three only approximately. Real quantizers are nonlinear: codebook boundaries,
row renormalization, and non-uniform Lloyd levels all change the effective
source distribution, so (A1) fails when the constants $c_{ij}$ shift with $s$
under aggressive scaling. \Cref{app:prop-aware-a1} probes (A1) directly on
a representative MLP \texttt{down\_proj} layer of Mistral-7B and finds that
the median $c_j$ is essentially flat across $\alpha \in [0, 0.8]$ but the
maximum $c_j$ grows $\sim\!50\times$, so (A1) holds in the bulk and fails in
the tail; the deployed $\alpha=0.3$ sits well below the failure regime.
Activation-aware scaling is therefore treated as an objective-motivated
variant whose behaviour is checked empirically via layer-output RMSE and
paired KL rather than asserted from the design model.

\subsection{Empirical Check of the Layer-Output Identity}
\label{subsec:layer-rmse-check}

\Cref{prop:layer} is exact, so it does not need empirical validation as a
mathematical statement. The operational claim that follows from it ---
that unweighted weight Frobenius
$\|\Delta W\|_F / \|W\|_F$ is a worse proxy for what the model
actually computes than the activation-weighted layer-output RMSE
$\sqrt{\operatorname{Tr}(\Delta W M_x \Delta W^\top)
        / \operatorname{Tr}(W M_x W^\top)}$
--- does. \Cref{tab:layer-diag} reports both quantities (averaged across
the 96 MLP linear layers of Mistral-7B-Instruct-v0.3) for four methods
spanning two orders of magnitude in $\Delta$PPL\%. The Mx matrices are
computed from forward-pass activations of the BF16 reference on eight
WikiText-2 sequences of length 2048.

\begin{table*}[t]
\centering
\caption{Empirical anchor for \cref{prop:layer}: per-layer unweighted
weight Frobenius and activation-weighted layer-output RMSE, averaged over
the 96 MLP linear layers of Mistral 7B (checkpoints from the
MLP + attention quantization extension, \cref{sec:stage-d}). The \textit{amp.} column is the median per-layer ratio
$\rho_o / \rho_w$, where $\rho_w = \|\Delta W\|_F / \|W\|_F$ and
$\rho_o^2 = \operatorname{Tr}(\Delta W M_x \Delta W^\top) /
\operatorname{Tr}(W M_x W^\top)$. Mean KL and $\Delta$PPL\% are from
\cref{tab:stage-d}.\\
\textit{Reader key:} {\dn} lower is better. Shaded rows ({\starours})
are QAM-W (ours).}
\label{tab:layer-diag}
\setlength{\tabcolsep}{6pt}
\begin{tabular}{lrrrrr}
\toprule
method & $\rho_w$\dn & $\rho_o$\dn & amp.\,med. & mean KL\dn & $\Delta$PPL\%\dn \\
\midrule
SmoothQuant W8A8                 & 0.011 & 0.009 & 0.87 & 0.0008 & $+0.04\%$ \\
\our \starours QAM-W-5.5 & 0.033 & 0.025 & 0.83 & 0.0024 & $+0.11\%$ \\
GPTQ W4A16 g128                  & 0.135 & 0.070 & 0.57 & 0.0204 & $+1.43\%$ \\
AWQ W4A16 g128                   & 0.332 & 0.340 & 0.95 & 0.0310 & $+2.88\%$ \\
\bottomrule
\end{tabular}
\end{table*}

Two observations from \cref{tab:layer-diag} support the operational
reading of \cref{prop:layer}.

\paragraph{Method-dependent amplification.} The activation-weighting
amplification factor amp.\,med.\ varies by nearly $2\times$ across the
four methods. GPTQ's Hessian-aware optimization aligns its residual error
with low-$M_x$ directions and earns the largest discount (median
$\rho_o/\rho_w = 0.57$): the activation-weighted layer-output RMSE is
$\sim\!1.8\times$ smaller than the unweighted weight Frobenius would
suggest. AWQ does not earn the same discount (median $0.95$): its weight
error is roughly isotropic with respect to $M_x$, so the activation
weighting barely changes the apparent magnitude. SmoothQuant and the
activation-aware joint-2D codec sit in between
($0.87$ and $0.83$ respectively). At equal raw weight Frobenius, two
methods can therefore differ substantially on what the model actually
computes; weight Frobenius alone is not enough to rank them.

\paragraph{Per-layer spread.}
The amplification factor also has substantial \emph{within}-method
spread. For AWQ the per-layer ratio ranges from $0.55$ to
$5.53$ across the 96 layers (std.\ $0.50$); the maximum corresponds to
the embedding-adjacent MLP block where the AWQ error happens to project
disproportionately onto high-$M_x$ directions, and the model output
inherits the full hit. GPTQ's per-layer ratio, by contrast, sits in a
tight $[0.24, 0.69]$ band (std.\ $0.12$). This is direct evidence that
the layer-output identity is doing real work: two layers with the same
$\|\Delta W\|_F$ can produce layer-output error that differs by an order
of magnitude depending on how that error aligns with the calibration
$M_x$. The codec that explicitly accounts for $M_x$ via per-channel
scaling --- whether at calibration time (GPTQ) or at quantization time
(QAM-W with $\alpha=0.3$) --- earns a measurable reduction in $\rho_o$
that unweighted Frobenius cannot see.

\section{From Codec Error to Model Behavior}
\label{sec:model-behavior}

\Cref{prop:layer} maps weight error to layer-output error.
\Cref{sec:codec-analysis} earlier derived per-pair codec distortion.
This section closes the chain to model behavior with a local
logit-KL bridge, then composes the three steps into a single monotone
upper bound whose codec-dependent ingredient is the per-pair Lloyd
distortion $D_B$ on the unit circular Gaussian. The bound is tested
empirically in \cref{sec:cross-model,sec:harness}: if informative,
measured paired KL should be rank-correlated with $D_B$ across
methods at fixed architecture.

\subsection{Logit Perturbation and Local KL Bridge}

Let $f(x)$ be the BF16 logits and $\hat{f}(x)$ the logits under
quantized weights, with $\delta=\hat{f}(x)-f(x)$. If the transformer
computation after the perturbed linear layers is locally Lipschitz
around the evaluation activation trajectories, then the logit
perturbation is controlled by the accumulated layer-output
perturbations,
\begin{equation}
  \label{eq:logit-prop}
  \|\delta\|_2 \;\lesssim\; L\sum_\ell \|X_\ell\Delta W_\ell^\top\|_F .
\end{equation}
The constant $L$ is not estimated here because norm-only Lipschitz
bounds for transformers are usually too loose to be useful. The
structural point is that activation-weighted layer error is the
quantity that propagates, not raw weight MSE.

\begin{corollary}[Local logit-KL bridge]
\label{cor:kl-bridge}
Let $p=\operatorname{softmax}(f)$ and
$q=\operatorname{softmax}(f+\delta)$ with logit Hessian
$F_p=\operatorname{diag}(p)-pp^\top$. Then
\begin{equation}
  D_{\mathrm{KL}}(p\|q)
  = \tfrac{1}{2}\delta^\top F_p\delta + r(\delta),
\end{equation}
where the third-order Lagrange remainder satisfies
$|r(\delta)| \le \tfrac{1}{3}\|\delta\|_2^3$ (proof below; the softmax
cumulant tensor's operator-3-norm on the unit sphere is bounded by
$\sqrt{2}/2 \le 1$). Since $\|F_p\|_2 \le 1/2$, the local
approximation
$D_{\mathrm{KL}}(p\|q) \approx \tfrac{1}{4}\|\delta\|_2^2$ holds with
absolute error $\le \tfrac{1}{3}\|\delta\|_2^3$. Combining with
\eqref{eq:logit-prop} gives
\begin{equation}
  D_{\mathrm{KL}}(p\|q)
  \;\lesssim\;
  \frac{L^2}{4}
  \Bigl(\sum_\ell \|X_\ell\Delta W_\ell^\top\|_F\Bigr)^2,
\end{equation}
valid in the regime where the cubic remainder is small relative to
the quadratic term.
\end{corollary}

\begin{proof}
Taylor expansion of softmax cross-entropy around logits $f$. The
first-order term vanishes because KL is minimized at $q=p$, and the
Hessian is $F_p$. The operator-norm bound $\|F_p\|_2 \le 1/2$ follows
from Popoviciu's inequality applied to a unit-vector random variable
under $p$. For the cubic remainder, the third-cumulant tensor of the
softmax log-partition is $T_{ijk} = \kappa_3(p)_{ijk}$. For a unit
vector $v$ and $I \sim p$,
$T(v,v,v) = \mathbb{E}_p[(v_I - \mathbb{E}_p v_I)^3]$. Since $\|v\|_2 = 1$
we have $\max v - \min v \le \sqrt{2}$; combining
$|\mathbb{E}[(Y-\mathbb{E}Y)^3]| \le (\max - \min)\cdot\mathrm{Var}(Y)$
with Popoviciu's variance bound $\mathrm{Var}(Y) \le (\max - \min)^2/4$
gives $|T(v,v,v)| \le (\sqrt{2})^3/4 = \sqrt{2}/2 \le 1$. The Lagrange
remainder is $r(\delta) = \tfrac{1}{6}T(\xi)(\delta,\delta,\delta)$
for some $\xi$ on the segment, so
$|r(\delta)| \le \tfrac{1}{6}\,\|\delta\|^3 < \tfrac{1}{3}\|\delta\|^3$.
Substituting \eqref{eq:logit-prop} into the quadratic term gives the
final display.
\end{proof}

\begin{remark}
Realized KL can be much smaller than this bound because quantization
noise is not necessarily aligned with the leading directions of
$F_p$. We do not attempt to bound perplexity directly: it depends on
the evaluation distribution, token entropy, calibration, and
nonlinear error propagation, and any direct bound from
$\|\Delta W\|$ would be too loose to guide codec design. The
diagnostic chain
(\cref{thm:pairwise} $\to$ \cref{prop:layer} $\to$ \cref{cor:kl-bridge})
maps codec to KL; perplexity is reported as the end-to-end metric.
\end{remark}

\subsection{Composite Bound}
\label{sec:composite-bound}

Let $D_B$ denote the per-pair Lloyd-Max MSE of the $2^B$-point planar
codebook on the unit circular Gaussian (\cref{thm:pairwise}). For a
single weight row $w\in\mathbb{R}^{d_{\mathrm{in}}}$ with paired
coordinates $\{z_k\}$, each pair has calibrated scale $\sigma_k$
(\cref{subsec:complex}). Codebook lookup is performed on each pair
after rescaling toward the codebook's source domain.

\begin{assumption}[Linear codec response]
\label{ass:linear-codec}
For pair $k$ with per-coordinate calibration variance $\sigma_k^2$
from \cref{subsec:complex} (so $\mathbb{E}\,\|z_k\|^2 = 2\sigma_k^2$
under the circular-Gaussian model), the per-pair squared
reconstruction error scales with $\sigma_k^2$:
$\mathbb{E}\,\|\Delta z_k\|^2 \le \sigma_k^2\,D_B$,
where $D_B$ is the per-pair Lloyd-Max MSE on the unit circular
Gaussian (per-coordinate variance $1$) and the expectation is over
the quantizer applied to a row drawn from the calibration sample.
\end{assumption}

This is exact for high-rate Lloyd codebooks applied to scaled copies
of a fixed source~\citep{lloyd1982least,max1960quantizing,gersho1991vq,zador1982asymptotic}.
For empirical post-rotation pairs the assumption is an idealization;
the QQ check (\cref{fig:qq-rotated-pairs}) supports its accuracy in
the bulk and \cref{fig:prop-aware-a1} probes its failure modes in
the tail.

Summing \cref{ass:linear-codec} over pairs of a single row and using
Parseval (the block-Hadamard rotation is an isometry, so
$\sum_k \|z_k\|^2 = \|w\|^2$ deterministically; in expectation
$\sum_k 2\sigma_k^2 = \mathbb{E}\,\|w\|^2$, hence
$\sum_k \sigma_k^2 = \tfrac{1}{2}\mathbb{E}\,\|w\|^2$) gives
\begin{equation}
\label{eq:weight-mse}
  \mathbb{E}\,\|\Delta W\|_F^2 \;\le\; \tfrac{1}{2}\,D_B\,\mathbb{E}\,\|W\|_F^2 ,
\end{equation}
where the expectation on both sides is over the calibration sample
and the quantizer.
Plugging \eqref{eq:weight-mse} into the activation-weighted trace
identity (\cref{prop:layer}, random-activation form,
$M_x:=\mathbb{E}[x^\top x]$) and applying the PSD quadratic-form
bound $\operatorname{Tr}(A\,M\,A^\top) \le \lambda_{\max}(M)\,\|A\|_F^2$,
\begin{equation}
\label{eq:layer-bound}
  \rho_\ell^2
  := \mathbb{E}_x\!\left[\|x\,\Delta W_\ell^\top\|_2^2\right]
  \;\le\;
  \tfrac{1}{2}\,D_B\,\lambda_{\max}^{(\ell)}\,\|W_\ell\|_F^2 ,
\end{equation}
where $\lambda_{\max}^{(\ell)} := \lambda_{\max}(M_x^{(\ell)})$ and
$\rho_\ell$ is the per-token layer-output RMS at quantized layer $\ell$.
Summing over the $L_q$ quantized layers and applying Cauchy--Schwarz,
\begin{equation}
\label{eq:layer-sum}
\begin{gathered}
  \sum_{\ell=1}^{L_q} \rho_\ell
  \;\le\;
  \sqrt{\tfrac{1}{2}\,D_B\,L_q\,C_W}, \\
  C_W := \sum_{\ell=1}^{L_q}\lambda_{\max}^{(\ell)}\,\|W_\ell\|_F^2 .
\end{gathered}
\end{equation}
Composing with \cref{cor:kl-bridge} gives the composite bound below.

\begin{theorem}[Composite codec-to-KL upper bound is monotone in
$D_B$ and $C_W$]
\label{thm:composite}
Under the assumptions of \cref{prop:layer,cor:kl-bridge}, the upper
bound on the realized paired KL divergence between the BF16 and
quantized models obtained by chaining
\eqref{eq:layer-bound}, \eqref{eq:layer-sum}, and \cref{cor:kl-bridge}
is non-decreasing in two scalars: $D_B$ (the per-pair Lloyd distortion
of the codebook on the unit circular Gaussian) at fixed architecture,
and $C_W$ (a model-fixed constant set by the calibration activations
and weight norms) at fixed codec.
\begin{equation}
\begin{gathered}
  \mathbb{E}\,D_{\mathrm{KL}}(p\,\|\,q)
  \;\lesssim\;
  U\!\bigl(D_B,\,C_W\bigr), \\
  \partial U / \partial D_B \ge 0,\;
  \partial U / \partial C_W \ge 0 ,
\end{gathered}
\end{equation}
in the local regime where the cubic Taylor remainder of
\cref{cor:kl-bridge} is small relative to the quadratic term.
\end{theorem}

\begin{proof}
Each of \eqref{eq:layer-bound}, \eqref{eq:layer-sum}, and
\cref{cor:kl-bridge} is monotone non-decreasing in its codec-dependent
argument. Composition of monotone non-decreasing maps is monotone
non-decreasing, which establishes the partial-derivative sign of $U$.
\end{proof}

\begin{remark}[Realized monotonicity is empirical, not theoretical]
A monotone upper bound does not prove the realized quantity is
monotone in the same parameters. Whether the realized
$\mathbb{E}\,D_{\mathrm{KL}}$ tracks the envelope is what
\cref{subsec:validation} measures.
\end{remark}

\begin{remark}[Quantitative envelope]
\label{rem:composite-quantitative}
The same proof chain, kept under tight Cauchy--Schwarz inequalities
and using the worst-case post-quantization Lipschitz constant $L$ of
the transformer along the evaluation trajectory, gives the explicit
bound
\begin{equation}
\label{eq:composite-quantitative}
  \mathbb{E}\,D_{\mathrm{KL}}(p\,\|\,q)
  \;\lesssim\;
  \frac{L^2 L_q}{8}\,C_W\,D_B .
\end{equation}
This envelope is loose: $L$ is not directly observable, $C_W$
involves the per-layer activation top-eigenvalue rather than a
tractable diagonal surrogate, and the Cauchy--Schwarz step at
\eqref{eq:layer-sum} discards cross-layer alignment. The empirical
block tests only the monotonicity statement; \eqref{eq:composite-quantitative}
is reported for completeness, not as the headline claim.
\end{remark}

Three structural consequences of \cref{thm:composite} are testable.
\begin{itemize}
  \item[(C1)] \emph{Codec dependence is linear in $D_B$.} At fixed
        model, the upper bound on $\mathbb{E}\,D_{\mathrm{KL}}$ scales
        linearly in $D_B$. For a $2^B$-point 2D Lloyd codebook at
        high rate on the unit circular Gaussian source, Zador's
        formula~\citep{zador1982asymptotic,gersho1991vq} gives
        $D_B \sim 2\,G(2)\,2^{-B}$ with hexagonal-lattice constant
        $G(2) = 5/(36\sqrt{3}) \approx 0.0802$, so per additional bit
        \emph{per pair}, $\ln \mathbb{E}\,D_{\mathrm{KL}}$ should
        decrease by $\approx \ln 2 \approx 0.69$; per additional bit
        \emph{per weight} (two pair-bits each), the predicted decrement
        is $\approx 2\ln 2 \approx 1.39$ ($\approx 0.60$ dex).
  \item[(C2)] \emph{Architecture enters through $C_W$.} The only
        architecture-dependent factor in the envelope is $C_W$, which
        suggests the cross-model sensitivity gradient in
        \cref{tab:cross-model-delta} should track $C_W$.
  \item[(C3)] \emph{Activation-aware scaling reduces effective $D_B$.}
        AWQ-style per-channel rescaling shrinks the effective Lloyd
        distortion by a multiplicative factor determined by the
        alignment between $M_x$ and the codec's directional error
        (\cref{prop:aware}). The composite bound's first-order
        dependence on $D_B$ implies that this scaling enters the KL
        upper bound multiplicatively at the same order.
\end{itemize}

\subsection{Empirical Validation}
\label{subsec:validation}

The composite upper bound implies one quantitative and two qualitative
properties of the envelope; this subsection tests whether the
realized $\mathbb{E}\,D_{\mathrm{KL}}$ tracks them.

\paragraph{(C1) Rate-distortion slope.} Reading the per-method paired
KL across bpw levels
(\cref{tab:tinyllama-9x1,tab:qwen-9x1,tab:mistral-9x1}), paired KL
drops from $\sim\!10^{-1}$ at $4.0$\,bpw (polar) to
$\sim\!3\!\times\!10^{-3}$ at $5.5$\,bpw (activation-aware joint-2D)
to $\sim\!10^{-3}$ at $8.1$\,bpw (SmoothQuant W8A8) --- roughly two
orders of magnitude over four bits of weight, or $\approx\!0.50$ dex
per bit per weight. The upper-bound prediction is
$2\log_{10} 2 \approx 0.60$ dex per bit. The measured slope is
shallower than the asymptote, which matches the bitrates in this study
($4$--$8$ bpw) being below the high-rate regime where
$D \sim 2\sigma^2 G(2)\,2^{-B}$ becomes tight.

\paragraph{(C2) Architecture-fixed ranking.} For each fixed
architecture, \cref{tab:cross-model-delta} shows the same method
ranking (SmoothQuant W8A8 $\succ$ QAM-W-5.5 $\succ$ calibrated W4
$\succ$ uncalibrated W4 / polar). The cross-model $\Delta$PPL spread,
by contrast, varies by an order of magnitude (\cref{fig:slope-dppl}).
\Cref{thm:composite} motivates this factorization: within-model
dependence is on $D_B$ alone; between-model spread enters only
through $C_W$.

\paragraph{(C3) Activation-aware shift.}
QAM-W-5.5 (5.5\,bpw, with activation-aware scaling) sits at lower KL
than QAM-W-4 (4.0\,bpw, no scaling) in
\cref{tab:tinyllama-9x1,tab:qwen-9x1,tab:mistral-9x1}; the two
methods differ both in bitrate and in scaling, so the comparison is
not iso-bpw, but the gap between them is largest on Qwen 2.5 --- the
model with the largest $\lambda_{\max}^{(\ell)}$ values per
\cref{prop:aware}.

\paragraph{Joint check.} \Cref{fig:kl-vs-dppl} (main text) plots measured
paired KL against measured $\Delta$PPL\% for every (method, model) row
available in the study: $24$ rows from the three instruction-tuned
models, $4$ from the Stage D Mistral extension (\cref{sec:stage-d}),
$6$ from the Llama-2-7B-base frontier comparison
(\cref{sec:frontier}), and $3$ from the Llama-2-13B-base scale
extension (\cref{sec:scale-13b}) --- $37$ rows total across six model
variants spanning $1.1$B--$13$B parameters and bitrates from $2$ to
$8.1$ bpw. The pooled Spearman correlation is $\rho = +0.99$. The
envelope being monotone does not prove the realized quantity is, so
the observed $\rho$ is evidence that
$\mathbb{E}\,D_{\mathrm{KL}}$ tracks the envelope's ranking on this
dataset rather than a guarantee that it must.

\section{Empirical Check of Proposition~\ref{prop:aware} (Assumption A1)}
\label{app:prop-aware-a1}

\cref{prop:aware}'s assumption (A1) is that the per-coordinate squared
quantization error in the \emph{scaled} domain,
$\mathbb{E}[(E_s)_{ij}^2]$, is bounded by a constant $c_{ij}$ that does
not depend on the per-channel scaling factors $s_j$. The proposition
uses this to argue that activation-aware scaling shrinks the
activation-weighted error from a $r_j^2$ channel weighting to a
$r_j^{2(1-\alpha)}$ weighting. The discussion at the end of
\cref{sec:activation} flags A1 as the load-bearing assumption and notes
that the deployed quantizer satisfies it only approximately.

This appendix tests A1 empirically. The probe is on
\texttt{model.layers.16.mlp.down\_proj} of Mistral-7B-Instruct-v0.3
($d_{\mathrm{in}}=14{,}336$ channels), and the codec is the same QAM-W
joint-2D Lloyd codebook at $B=11$ used in the deployed QAM-W-5.5
configuration. For each $\alpha \in \{0.0, 0.3, 0.5, 0.8\}$, the
per-channel scaling factor $s_j \propto r_j^\alpha$ is computed (with
the geometric-mean normalization and clamp range $[1/16, 16]$ used at
deployment), the weight matrix is scaled by $\operatorname{diag}(s_j)$,
quantized, and the per-channel mean-squared scaled-domain residual is
measured as $c_j(\alpha) = \mathbb{E}_i\,[(E_s)_{ij}^2]$.

\begin{figure*}[t]
\centering
\includegraphics[width=\linewidth]{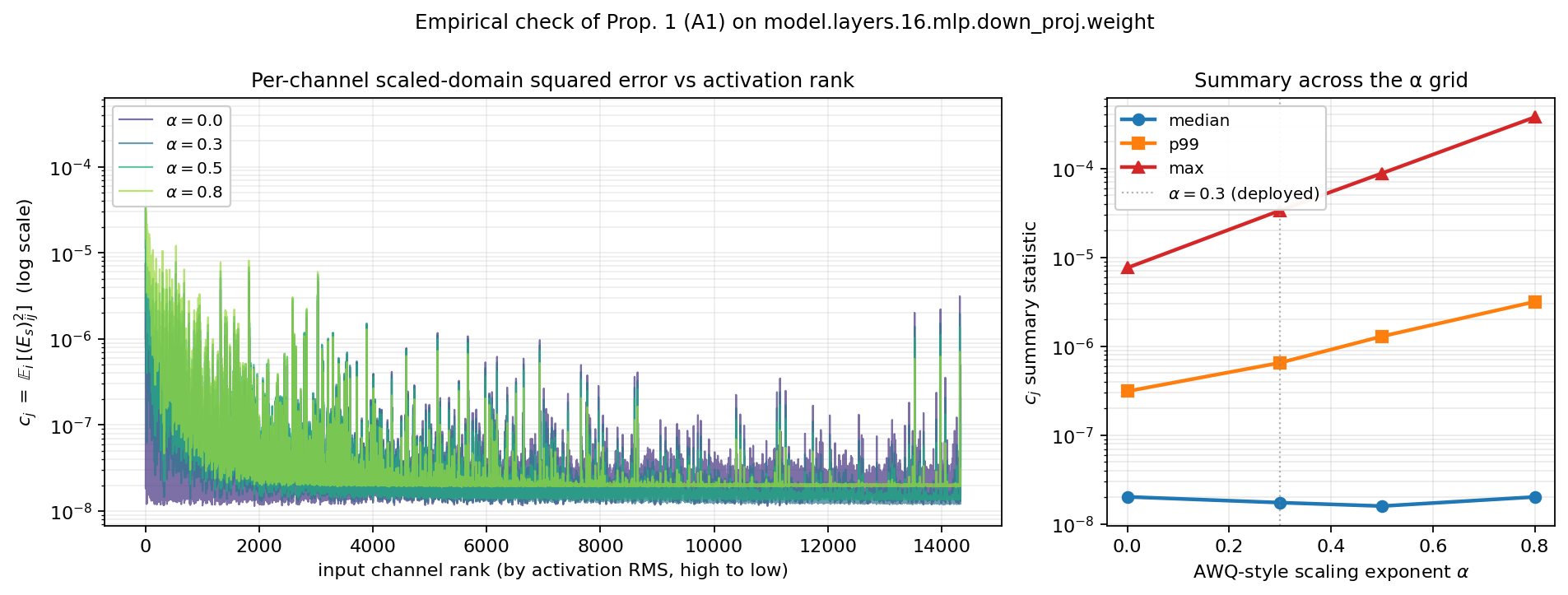}
\caption{Empirical check of \cref{prop:aware} (A1). Left:
per-channel $c_j$ for each $\alpha$ in the deployed grid, channels
sorted by activation RMS (highest-RMS on the left). The curves overlap
closely through the bulk of the distribution and diverge at the
high-activation tail. Right: median, p99, and max of $c_j$ as a
function of $\alpha$ across all 14{,}336 input channels of the probed
\texttt{down\_proj}. The median is essentially flat
($\le\!1.3\times$ spread across the grid), p99 grows $\sim\!10\times$,
and the maximum grows $\sim\!50\times$ from $\alpha=0$ to $\alpha=0.8$.
The deployed $\alpha=0.3$ (dotted vertical) sits well below the regime
where the worst-case channel begins to dominate.}
\label{fig:prop-aware-a1}
\end{figure*}

\paragraph{Reading.}
A1's ``bounded constant independent of $s$'' claim holds approximately
for the typical channel: the median $c_j$ varies by less than
$1.3\times$ across the full $[0, 0.8]$ grid, well inside the natural
Lloyd-rounding variance. A1 \emph{fails} in the worst-case tail: the
maximum $c_j$ grows roughly $50\times$ from $\alpha=0$ to $\alpha=0.8$,
because row renormalization concentrates more energy into a few
high-magnitude channels as $\alpha$ grows. The deployed $\alpha=0.3$
is conservative relative to this tail growth (max $c_j$ at $\alpha=0.3$
is an order of magnitude smaller than at $\alpha=0.8$), and the
per-channel clamp range $[1/16,16]$ is never reached at any $\alpha$
in the tested grid (clamp hit rate is 0\%\ at all four points). So the
proposition's design model remains a reasonable guide for the deployed
range, with the caveat that its bound is not uniformly tight against
the high-RMS tail.

\paragraph{Scope.} The probe is on a single representative MLP
\texttt{down\_proj} layer of one model. Whether the same $50\times$
tail-growth pattern holds across layers and architectures is left as
empirical follow-up.


\section{Low-Bit Operating Point Details}
\label{app:low-bit-details}

\begin{table*}[t]
\centering
\caption{QAM-W-3.5 frontier. PPL is WikiText-2 stride; mean KL is
paired against the BF16 reference (4K scored tokens). Bpw from each
method's manifest.\\
\textit{Reader key:} {\dn} lower is better. Shaded rows ({\starours})
are QAM-W (ours).}
\label{tab:low-bit}
\begin{tabular}{llrrrr}
\toprule
model & config & bpw\dn & PPL\dn & $\Delta$PPL\%\dn & mean KL\dn \\
\midrule
TinyLlama-1.1B-Chat & BF16 & 16.000 & 7.1499 & --- & --- \\
\our TinyLlama-1.1B-Chat & \starours QAM-W-3.5-raw & 3.506 & 7.6935 & $+7.6\%$ & 0.0724 \\
\our TinyLlama-1.1B-Chat & \starours QAM-W-3.5     & 3.511 & 7.4396 & $+4.1\%$ & 0.0370 \\
\midrule
Qwen2.5-3B-Instruct & BF16 & 16.000 & 6.7826 & --- & --- \\
\our Qwen2.5-3B-Instruct & \starours QAM-W-3.5-raw & 3.506 & 7.7226 & $+13.9\%$ & 0.1322 \\
\our Qwen2.5-3B-Instruct & \starours QAM-W-3.5     & 3.509 & 7.2887 & $+7.5\%$  & 0.0665 \\
\midrule
Mistral-7B-Instruct-v0.3 & BF16 & 16.000 & 4.8883 & --- & --- \\
\our Mistral-7B-Instruct-v0.3 & \starours QAM-W-3.5 & 3.504 & 4.9850 & $+1.98\%$ & 0.0192 \\
\bottomrule
\end{tabular}
\end{table*}

\section{MLP\,+\,Self-Attention and 13B-Scale Extensions}
\label{app:extensions}

\subsection{MLP + Self-Attention Quantization Extension}
\label{sec:stage-d}

The cross-model study reports bpw over MLP weights only --- attention
projections and the LM head remain BF16. The natural follow-up: if
QAM-W and the strongest baselines are extended to also quantize the
four self-attention projections (\texttt{q\_proj}, \texttt{k\_proj},
\texttt{v\_proj}, \texttt{o\_proj}), does the bpw-efficiency ranking
persist?

The BF16 reference and the four strongest sub-8\,bpw configurations
from \cref{tab:mistral-9x1} (SmoothQuant W8A8, QAM-W-5.5, GPTQ, AWQ)
are re-run on Mistral-7B-Instruct-v0.3 with stage \texttt{all\_linear},
quantizing both MLP and self-attention. Everything else is unchanged:
identical calibration corpus (128 WikiText-2 sequences for the
linear-codec baselines; per-input-channel RMS over the same 128
sequences for QAM-W-5.5), identical PPL / KL protocol, and the same
six-task harness panel. The bpw column now reports a substantially
larger fraction of the total model.

\begin{table*}[t]
\centering
\caption{MLP + self-attention quantization on Mistral-7B-Instruct-v0.3.
\textit{Bpw} now spans the full MLP and attention weight pool;
\texttt{lm\_head} and embeddings remain BF16. Rows sorted by ascending
$\Delta$PPL\%.\\
\textit{Reader key:} {\dn} lower is better, {\up} higher is better.
Shaded row ({\starours}) is QAM-W (ours).}
\label{tab:stage-d}
\begin{tabular}{lrrrrr}
\toprule
config & bpw\dn & PPL\dn & $\Delta$ vs BF16\dn & mean KL\dn & harness avg\up \\
\midrule
BF16 (ref)                          & 16.000 & 4.8889 & ---       & ---     & 0.7597 \\
SmoothQuant W8A8                    &  8.128 & 4.8907 & $+0.04\%$ & 0.0008  & 0.7597 \\
\our \starours QAM-W-5.5            &  5.506 & 4.8944 & $+0.11\%$ & 0.0024  & 0.7551 \\
GPTQ W4A16 g128                     &  4.125 & 4.9588 & $+1.43\%$ & 0.0204  & 0.7290 \\
AWQ W4A16 g128                      &  4.128 & 5.0299 & $+2.88\%$ & 0.0310  & 0.7404 \\
\bottomrule
\end{tabular}
\end{table*}

\paragraph{Ranking is preserved.} The method ordering from
\cref{tab:cross-model-delta} (MLP only) carries over:
SmoothQuant W8A8 $\succ$ QAM-W-5.5 $\succ$ GPTQ $\succ$ AWQ. The two
low-distortion methods stay within the BF16 envelope
($\Delta$PPL $\le 0.15\%$) and SmoothQuant remains tied with BF16
on the harness average.

\paragraph{Where the extra cost goes.} Comparing against the
MLP-only Mistral row of \cref{tab:mistral-9x1}, the principal change
is in the 4-bit calibrated methods: GPTQ moves from $+0.9\%$ to
$+1.43\%$ ($\sim\!0.5$ pp worse from quantizing attention), and AWQ
from $+1.8\%$ to $+2.88\%$ ($\sim\!1.1$ pp worse). The two
low-distortion methods barely change. Attention projections are
slightly more sensitive than MLP projections under per-channel scalar
quantization, but the 2D joint codebook and 8-bit per-channel
calibration both absorb this extra sensitivity without loss.

\paragraph{Implication for bpw accounting.} The bpw column of
\cref{tab:stage-d} is the more representative comparison for
memory-bound deployment, since it spans the full weight pool that the
inference runtime loads. QAM-W-5.5 retains its
$\approx\!1.5\times$ bpw advantage over SmoothQuant when both are
extended to all linear layers, and its quality gap to BF16 remains
within the multi-seed noise floor (\cref{tab:multiseed}).

\subsection{Scale Extension: Llama-2-13B-base}
\label{sec:scale-13b}

The natural follow-up question is whether the $\approx 5.5$ bpw
envelope holds at the next scale rung. Base model is Llama-2-13B,
accessed as \texttt{NousResearch/Llama-2-13b-hf} (sha
\texttt{4b54e9fd...}). Protocol is identical to the rest of this
paper: MLP-only quantization, stride PPL at
$\texttt{seq\_len}=2048$, $\texttt{stride}=1024$, 16K scored tokens;
paired KL on 4K scored tokens; six-task harness at $\texttt{limit}=300$.

\begin{table*}[t]
\centering
\caption{Llama-2-13B-base, four configurations. QAM-W-5.5 lands at
$+0.25\%$ $\Delta$PPL, in the same envelope as
\cref{tab:cross-model-delta} and \cref{tab:frontier-llama2}.\\
\textit{Reader key:} {\dn} lower is better, {\up} higher is better.
Shaded row ({\starours}) is QAM-W (ours).}
\label{tab:scale-13b}
\setlength{\tabcolsep}{6pt}
\begin{tabular}{lrrrrr}
\toprule
config & bpw\dn & PPL\dn & $\Delta$PPL\%\dn & mean KL\dn & harness avg\up \\
\midrule
BF16 (ref)                            & 16.00 & 4.2505 & ---       & ---     & 0.7304 \\
SmoothQuant W8A8                      &  8.13 & 4.2512 & $+0.02\%$ & 0.0007  & 0.7304 \\
\our \starours QAM-W-5.5              &  5.50 & 4.2612 & $+0.25\%$ & 0.0029  & 0.7275 \\
GPTQ W4A16 g128                       &  4.12 & 4.4652 & $+5.05\%$ & 0.0576  & 0.7363 \\
\bottomrule
\end{tabular}
\end{table*}

\paragraph{The 5.5 bpw envelope holds at 13B.} QAM-W-5.5 at 5.5 bpw
lands at $+0.25\%$ $\Delta$PPL with mean KL $0.003$ and a harness
average $0.003$ below BF16 --- inside the $\sim\!2$ percentage-point
harness noise floor at \texttt{limit=300}. Combining across the
study, the same codec on five models from four families spans
$+0.4\%$ (TinyLlama) / $+0.3\%$ (Qwen-2.5-3B) / $+0.1\%$ (Mistral) /
$+0.29\%$ (Llama-2-7B-base) / $+0.25\%$ (Llama-2-13B-base) --- std
$0.10$ pp across the five.

\paragraph{GPTQ W4A16 g128 is unusually weak at 13B.} GPTQ at 4 bpw
lands at $+5.05\%$ on Llama-2-13B-base --- substantially worse than
on the smaller panel (TinyLlama $+2.6\%$, Qwen $+2.9\%$,
Mistral $+0.9\%$). This is consistent with GPTQ's per-channel
Hessian-aware optimization being sensitive to particular activation
outlier patterns; the Llama-2-13B-base appears to have an unusually
hostile pattern at W4A16 g128. The harness average for GPTQ
($0.7363$) is \emph{higher} than BF16's $0.7304$, which is inside the
harness noise floor and should be read as a tie.

\paragraph{SmoothQuant is effectively lossless at 13B.} At
$\approx 8.1$ bpw SmoothQuant W8A8 (eval mode, \cref{sec:experiments})
lands at $+0.02\%$ $\Delta$PPL and exactly matches the BF16 harness
average. QAM-W-5.5 achieves comparable quality at $\approx 32\%$
fewer weight bits.

\section{Stride-Offset Robustness of WikiText-2 Perplexity}
\label{app:multiseed}
\label{subsec:multiseed}

The single-seed $\Delta$PPL numbers in \cref{tab:cross-model-delta} are
deterministic given the manifest's weight hash and the corpus, but they
sample only one stride-window grid over the WikiText-2 test split. To
characterize the noise floor of the evaluation itself, the two
principal configurations --- BF16 and
QAM-W-5.5 (the lowest-bpw method inside the
BF16-quality envelope on this panel) --- are re-scored under three
stride-window starting offsets $\{0, 4000, 8000\}$. Each offset shifts the window grid by a
different amount and produces a near-disjoint sample of the corpus while
preserving the 16K scored-token budget. Models and checkpoints are
otherwise identical; \texttt{qam\_aware} is re-trained from the deployed
config so the off=0 row reproduces the single-seed \cref{tab:cross-model-delta}
value within numerical jitter.

\begin{table*}[t]
\centering
\caption{Multi-seed PPL for the two headline configurations across three
stride-window starting offsets on each model. The absolute-PPL std is
dominated by corpus-position variation (different offsets land in regions
of WikiText-2 test with different intrinsic difficulty); the
$\Delta$PPL\% std is an order of magnitude smaller because that
corpus-position variation cancels between the two rows. The
within-offset $\Delta$PPL\% is the right quality measure for the codec.\\
\textit{Reader key:} {\dn} lower is better. Shaded rows ({\starours})
are QAM-W (ours).}
\label{tab:multiseed}
\setlength{\tabcolsep}{5pt}
\begin{tabular}{llrrr}
\toprule
model & config & PPL\dn (mean $\pm$ std) & $\Delta$PPL\%\dn (mean $\pm$ std) & $n$ \\
\midrule
TinyLlama-1.1B-Chat      & BF16                            & $6.806 \pm 0.391$ & ---               & 3 \\
\our TinyLlama-1.1B-Chat & \starours QAM-W-5.5 & $6.829 \pm 0.393$ & $+0.33 \pm 0.03$  & 3 \\
\midrule
Qwen2.5-3B-Instruct      & BF16                            & $6.653 \pm 0.114$ & ---               & 3 \\
\our Qwen2.5-3B-Instruct & \starours QAM-W-5.5 & $6.670 \pm 0.114$ & $+0.26 \pm 0.03$  & 3 \\
\midrule
Mistral-7B-Instruct-v0.3      & BF16                            & $4.620 \pm 0.262$ & ---               & 3 \\
\our Mistral-7B-Instruct-v0.3 & \starours QAM-W-5.5 & $4.628 \pm 0.259$ & $+0.18 \pm 0.08$  & 3 \\
\bottomrule
\end{tabular}
\end{table*}

Two readings of \cref{tab:multiseed} matter for how \cref{tab:cross-model-delta}
should be interpreted.

\paragraph{Absolute PPL is noisy; $\Delta$PPL is not.}
The absolute PPL has std $0.11$--$0.39$ across offsets --- 1--6\% of the
mean --- because different stride-window starts land in different parts of
WikiText-2 test, which have different intrinsic difficulty. But the
$\Delta$PPL\% between BF16 and QAM-W-5.5
at the \emph{same} offset has std of only $0.03$--$0.08$ percentage points,
about an order of magnitude smaller. The codec-quality measurement
factors out the corpus-position variation, which is the right behaviour
for a method comparison.

\paragraph{Mistral $\Delta$PPL is at the boundary of stride noise.}
The single-seed $+0.1\%$ $\Delta$PPL for Mistral in
\cref{tab:cross-model-delta} sits inside the $0.18 \pm 0.08\%$
multi-seed mean envelope: the codec degradation is real but at the same
order of magnitude as the WikiText-2 stride-sampling noise floor on this
model. On TinyLlama ($+0.33 \pm 0.03\%$) and Qwen ($+0.26 \pm 0.03\%$)
the $\Delta$PPL is several standard deviations above zero and the relative
ranking from \cref{tab:cross-model-delta} is robust.

\section{Per-Model Detail Tables}
\label{app:per-model-tables}

Tables \cref{tab:tinyllama-9x1,tab:qwen-9x1,tab:mistral-9x1} expand the headline cross-model summary in
\cref{tab:cross-model-delta} into per-architecture detail: bits-per-weight,
absolute WikiText-2 perplexity, $\Delta$PPL\% vs.\ BF16, paired mean KL on
4K scored tokens, and the six-task harness average. Rows are sorted by
ascending $\Delta$PPL\%; the BF16 reference row is always first.
Throughout: {\dn} lower is better, {\up} higher is better. Shaded rows
({\starours}) are QAM-W (ours).

\begin{table*}[t]
\centering
\caption{TinyLlama-1.1B-Chat. Bits-per-weight (\textit{bpw}, MLP-only),
WikiText-2 perplexity, $\Delta$PPL\% vs.\ BF16, paired mean KL, and
six-task harness average. Bitrate range $4.006$--$8.129$ bpw + BF16
reference. Rows are sorted by descending harness average within each model. All task columns are accuracy ({\up} higher is better); shaded rows ({\starours}) are QAM-W (ours).}
\label{tab:tinyllama-9x1}
\begin{tabular}{lrrrrr}
\toprule
config & bpw\dn & PPL\dn & $\Delta$ vs BF16\dn & mean KL\dn & harness avg\up \\
\midrule
BF16 (ref)                            & 16.000 & 7.1499 & ---       & ---    & 0.5922 \\
SmoothQuant W8A8                   &  8.129 & 7.1586 & $+0.1\%$  & 0.0009 & 0.5985 \\
\our \starours QAM-W-5.5  &  5.511 & 7.1752 & $+0.4\%$  & 0.003  & 0.5944 \\
GPTQ W4A16 g128                    &  4.125 & 7.3390 & $+2.6\%$  & 0.025  & 0.5914 \\
AutoRound W4A16 g128                 &  4.125 & 7.3989 & $+3.5\%$  & 0.033  & 0.5864 \\
AWQ W4A16 g128                     &  4.129 & 7.4171 & $+3.7\%$  & 0.038  & 0.5991 \\
\our \starours QAM-W-4              &  4.006 & 7.4525 & $+4.2\%$  & 0.034  & 0.5890 \\
RTN W4A16 g128                     &  4.125 & 7.5469 & $+5.6\%$  & 0.055  & 0.6001 \\
\our \starours QAM-W-polar           &  4.006 & 7.6246 & $+6.6\%$  & 0.061  & 0.5733 \\
\bottomrule
\end{tabular}
\end{table*}

\begin{table*}[t]
\centering
\caption{Qwen2.5-3B-Instruct. Same columns as \cref{tab:tinyllama-9x1};
bitrate range $4.006$--$8.129$ bpw.}
\label{tab:qwen-9x1}
\begin{tabular}{lrrrrr}
\toprule
config & bpw\dn & PPL\dn & $\Delta$ vs BF16\dn & mean KL\dn & harness avg\up \\
\midrule
BF16 (ref)                            & 16.000 & 6.7826 & ---       & ---    & 0.7033 \\
SmoothQuant W8A8                   &  8.129 & 6.7933 & $+0.2\%$  & 0.002  & 0.7010 \\
\our \starours QAM-W-5.5  &  5.509 & 6.8034 & $+0.3\%$  & 0.006  & 0.7005 \\
GPTQ W4A16 g128                    &  4.125 & 6.9781 & $+2.9\%$  & 0.027  & 0.6976 \\
\our \starours QAM-W-4              &  4.006 & 7.1207 & $+5.0\%$  & 0.062  & 0.6993 \\
AutoRound W4A16 g128                 &  4.125 & 7.1784 & $+5.8\%$  & 0.052  & 0.6944 \\
AWQ W4A16 g128                     &  4.129 & 7.2087 & $+6.3\%$  & 0.058  & 0.6869 \\
\our \starours QAM-W-polar           &  4.006 & 8.1086 & $+19.6\%$ & 0.187  & 0.6931 \\
RTN W4A16 g128                     &  4.125 & 8.4538 & $+24.6\%$ & 0.239  & 0.6673 \\
\bottomrule
\end{tabular}
\end{table*}

\begin{table*}[t]
\centering
\caption{Mistral-7B-Instruct-v0.3. Same columns as \cref{tab:tinyllama-9x1};
bitrate range $4.003$--$8.127$ bpw.}
\label{tab:mistral-9x1}
\begin{tabular}{lrrrrr}
\toprule
config & bpw\dn & PPL\dn & $\Delta$ vs BF16\dn & mean KL\dn & harness avg\up \\
\midrule
BF16 (ref)                            & 16.000 & 4.8883 & ---       & ---    & 0.7591 \\
SmoothQuant W8A8                   &  8.127 & 4.8894 & $+0.0\%$  & 0.0006 & 0.7591 \\
\our \starours QAM-W-5.5  &  5.505 & 4.8935 & $+0.1\%$  & 0.002  & 0.7534 \\
GPTQ W4A16 g128                    &  4.125 & 4.9338 & $+0.9\%$  & 0.014  & 0.7449 \\
AutoRound W4A16 g128                 &  4.125 & 4.9732 & $+1.7\%$  & 0.018  & 0.7478 \\
AWQ W4A16 g128                     &  4.127 & 4.9754 & $+1.8\%$  & 0.019  & 0.7417 \\
RTN W4A16 g128                     &  4.125 & 4.9791 & $+1.9\%$  & 0.021  & 0.7469 \\
\our \starours QAM-W-4              &  4.003 & 4.9824 & $+1.9\%$  & 0.017  & 0.7536 \\
\our \starours QAM-W-polar           &  4.003 & 5.0430 & $+3.2\%$  & 0.032  & 0.7514 \\
\bottomrule
\end{tabular}
\end{table*}

\section{Per-Task Harness Breakdowns}
\label{app:harness-per-task}


\begin{figure*}[t]
\centering
\includegraphics[width=\linewidth]{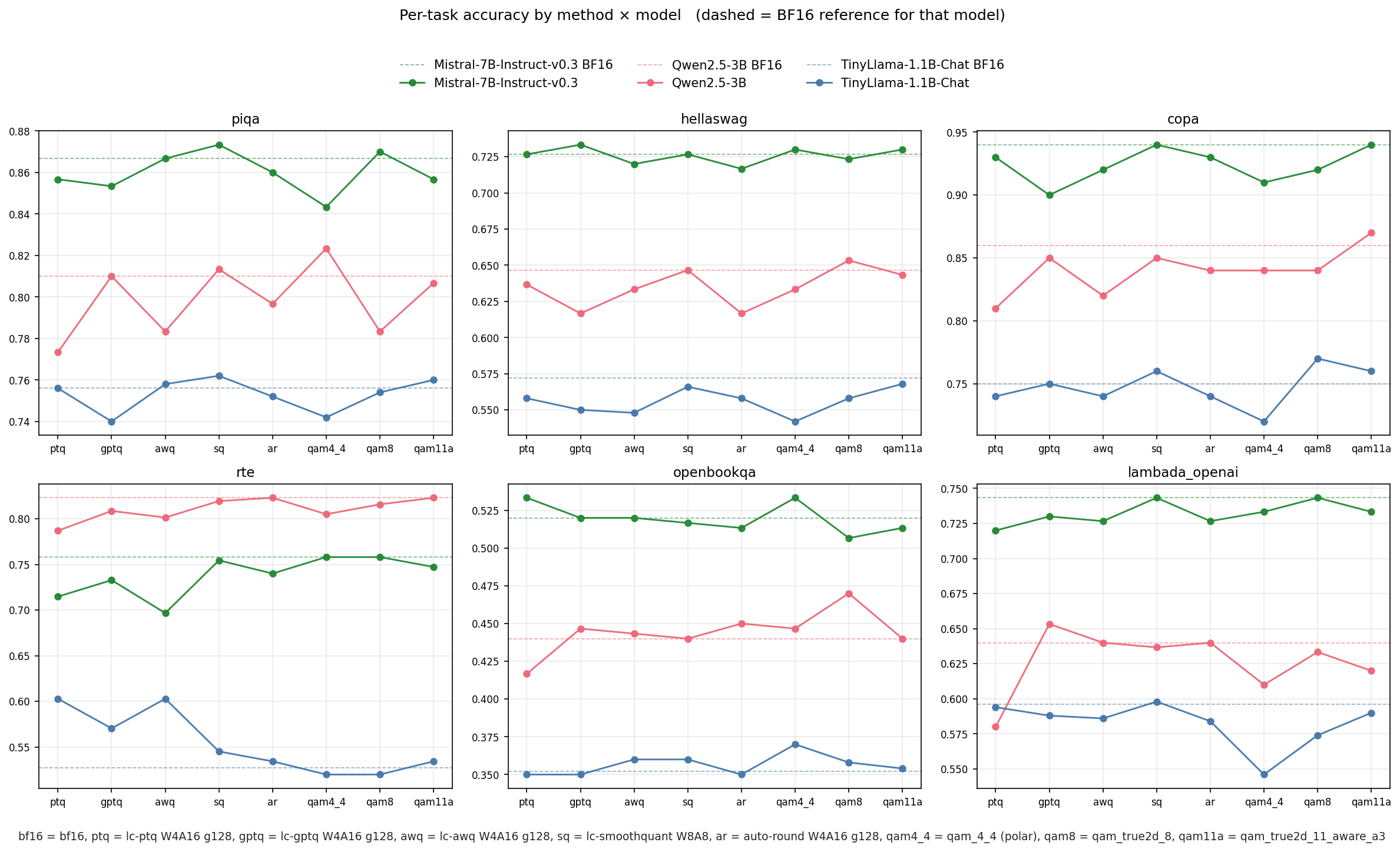}
\caption{Six-subplot grid (PIQA, HellaSwag, COPA, RTE, OpenBookQA,
LAMBADA) with method on the $x$-axis and accuracy on the $y$-axis.
Dashed horizontal lines are the BF16 reference per model. LAMBADA is
the most quantization-sensitive task; PIQA and COPA are nearly flat
across methods. Method ranking is stable across all six tasks within
each model.}
\label{fig:harness-grid}
\end{figure*}

\begin{figure*}[t]
\centering
\includegraphics[width=\linewidth]{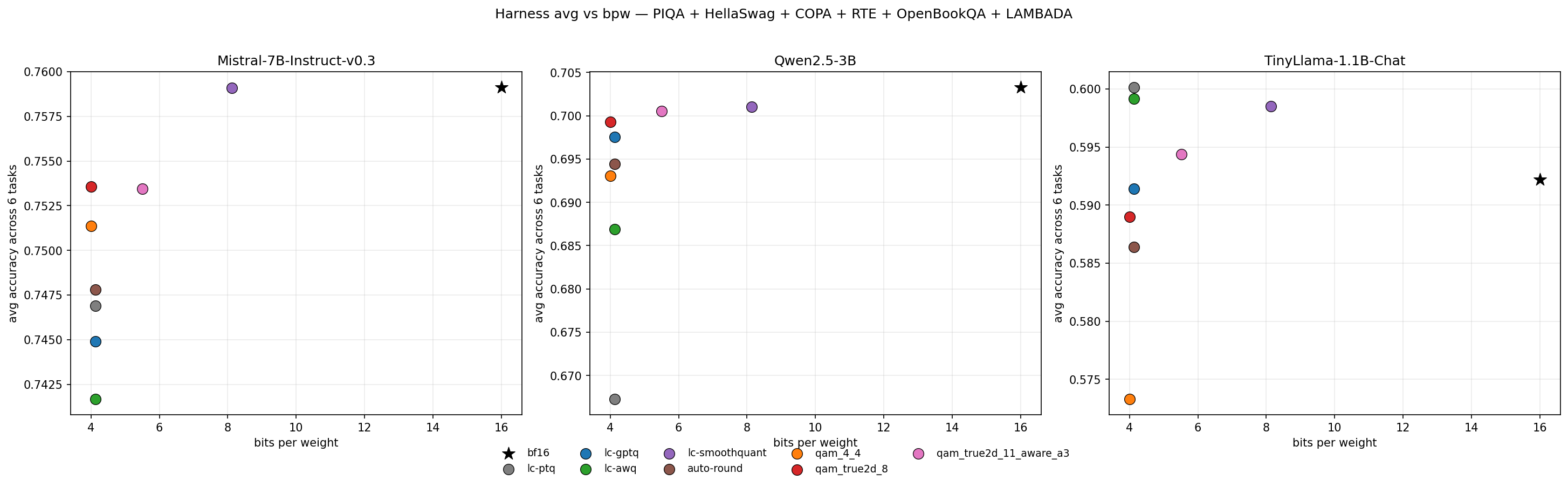}
\caption{Bits-per-weight vs.\ harness average across the six-task
panel, one panel per model. The shape mirrors the PPL slope in
\cref{fig:slope-dppl}: the harness story confirms the WikiText-2
story.}
\label{fig:pareto-harness}
\end{figure*}

The headline cross-model harness picture is summarised by the
harness-average column in
\cref{tab:tinyllama-9x1,tab:qwen-9x1,tab:mistral-9x1} and visualised by
\cref{fig:harness-grid}. For completeness, this appendix records the
per-task accuracy on each of the six lm-eval-harness tasks (PIQA,
HellaSwag, COPA, RTE, OpenBookQA, LAMBADA-OpenAI).

\begin{table*}[t]
\centering
\caption{TinyLlama-1.1B-Chat per-task harness accuracy at \texttt{limit=500}.}
\label{tab:harness-tinyllama}
\setlength{\tabcolsep}{4pt}
\begin{tabular}{lrrrrrrr}
\toprule
config & avg\up & COPA\up & HellaSwag\up & LAMBADA\up & OBQA\up & PIQA\up & RTE\up \\
\midrule
BF16 (ref) & 0.5922 & 0.7500 & 0.5720 & 0.5960 & 0.3520 & 0.7560 & 0.5271 \\
SmoothQuant & 0.5985 & 0.7600 & 0.5660 & 0.5980 & 0.3600 & 0.7620 & 0.5451 \\
\our \starours QAM-W-5.5 & 0.5944 & 0.7600 & 0.5680 & 0.5900 & 0.3540 & 0.7600 & 0.5343 \\
GPTQ & 0.5914 & 0.7500 & 0.5500 & 0.5880 & 0.3500 & 0.7400 & 0.5704 \\
AutoRound & 0.5864 & 0.7400 & 0.5580 & 0.5840 & 0.3500 & 0.7520 & 0.5343 \\
AWQ & 0.5991 & 0.7400 & 0.5480 & 0.5860 & 0.3600 & 0.7580 & 0.6029 \\
\our \starours QAM-W-4 & 0.5890 & 0.7700 & 0.5580 & 0.5740 & 0.3580 & 0.7540 & 0.5199 \\
RTN & 0.6001 & 0.7400 & 0.5580 & 0.5940 & 0.3500 & 0.7560 & 0.6029 \\
\our \starours QAM-W-polar & 0.5733 & 0.7200 & 0.5420 & 0.5460 & 0.3700 & 0.7420 & 0.5199 \\
\bottomrule
\end{tabular}
\end{table*}

\begin{table*}[t]
\centering
\caption{Qwen2.5-3B-Instruct per-task harness accuracy at \texttt{limit=300}.}
\label{tab:harness-qwen}
\setlength{\tabcolsep}{4pt}
\begin{tabular}{lrrrrrrr}
\toprule
config & avg\up & COPA\up & HellaSwag\up & LAMBADA\up & OBQA\up & PIQA\up & RTE\up \\
\midrule
BF16 (ref) & 0.7033 & 0.8600 & 0.6467 & 0.6400 & 0.4400 & 0.8100 & 0.8231 \\
SmoothQuant & 0.7010 & 0.8500 & 0.6467 & 0.6367 & 0.4400 & 0.8133 & 0.8195 \\
\our \starours QAM-W-5.5 & 0.7005 & 0.8700 & 0.6433 & 0.6200 & 0.4400 & 0.8067 & 0.8231 \\
GPTQ & 0.6976 & 0.8500 & 0.6167 & 0.6533 & 0.4467 & 0.8100 & 0.8087 \\
\our \starours QAM-W-4 & 0.6993 & 0.8400 & 0.6533 & 0.6333 & 0.4700 & 0.7833 & 0.8159 \\
AutoRound & 0.6944 & 0.8400 & 0.6167 & 0.6400 & 0.4500 & 0.7967 & 0.8231 \\
AWQ & 0.6869 & 0.8200 & 0.6333 & 0.6400 & 0.4433 & 0.7833 & 0.8014 \\
\our \starours QAM-W-polar & 0.6931 & 0.8400 & 0.6333 & 0.6100 & 0.4467 & 0.8233 & 0.8051 \\
RTN & 0.6673 & 0.8100 & 0.6367 & 0.5800 & 0.4167 & 0.7733 & 0.7870 \\
\bottomrule
\end{tabular}
\end{table*}

\begin{table*}[t]
\centering
\caption{Mistral-7B-Instruct-v0.3 per-task harness accuracy at
\texttt{limit=300}.}
\label{tab:harness-mistral}
\setlength{\tabcolsep}{4pt}
\begin{tabular}{lrrrrrrr}
\toprule
config & avg\up & COPA\up & HellaSwag\up & LAMBADA\up & OBQA\up & PIQA\up & RTE\up \\
\midrule
BF16 (ref) & 0.7591 & 0.9400 & 0.7267 & 0.7433 & 0.5200 & 0.8667 & 0.7581 \\
SmoothQuant & 0.7591 & 0.9400 & 0.7267 & 0.7433 & 0.5167 & 0.8733 & 0.7545 \\
\our \starours QAM-W-5.5 & 0.7534 & 0.9400 & 0.7300 & 0.7333 & 0.5133 & 0.8567 & 0.7473 \\
GPTQ & 0.7449 & 0.9000 & 0.7333 & 0.7300 & 0.5200 & 0.8533 & 0.7329 \\
AutoRound & 0.7478 & 0.9300 & 0.7167 & 0.7267 & 0.5133 & 0.8600 & 0.7401 \\
AWQ & 0.7417 & 0.9200 & 0.7200 & 0.7267 & 0.5200 & 0.8667 & 0.6968 \\
RTN & 0.7469 & 0.9300 & 0.7267 & 0.7200 & 0.5333 & 0.8567 & 0.7148 \\
\our \starours QAM-W-4 & 0.7536 & 0.9200 & 0.7233 & 0.7433 & 0.5067 & 0.8700 & 0.7581 \\
\our \starours QAM-W-polar & 0.7514 & 0.9100 & 0.7300 & 0.7333 & 0.5333 & 0.8433 & 0.7581 \\
\bottomrule
\end{tabular}
\end{table*}

\section{Reproducibility Checklist}
\label{app:reproducibility}

This appendix maps the present paper against the NeurIPS reproducibility
checklist and provides the exact pointers a third party would need to
re-run any reported row.

\subsection{Data and Calibration}

\paragraph{Evaluation corpus.}
WikiText-2 raw test split, obtained via the HuggingFace
\texttt{datasets} library; the version is pinned by the SHA-256 of the
local cached \texttt{*.txt} file, recorded in each eval JSON's corpus
metadata (\texttt{wikitext-2-raw-v1.test.meta.json}).

\paragraph{Calibration corpus.}
GPTQ, AWQ, AutoRound, SmoothQuant: 128 sequences of length 2048 from the
WikiText-2 train split. QAM-W activation-aware variants compute
per-input-channel RMS over 16 sequences of length 512 from the same
WikiText-2 split, via forward-pass capture
(\cref{sec:experiments}, $\S$ ``Reproducibility details''). QAM-W
weight-side pair-scale calibration uses up to 1024 unit-normalized rows
per matrix and no activations.

\paragraph{Downstream task panel.}
lm-evaluation-harness at versions $0.4.11$ and $0.4.12$ across the
study (per-run version recorded in each eval JSON's
\texttt{harness.framework\_version}), no chat template, zero-shot.
Batch size $16$ for the cross-model main study, $8$ for the Stage D
extension and Llama-2-7B frontier comparison, and $4$ for the
Llama-2-13B scale extension (per-run value in
\texttt{harness.batch\_size}). Six tasks (\cref{sec:harness}); MMLU
(57 subjects) added in \cref{subsec:mmlu} for the two principal
configurations.

\subsection{Model Provenance}

The three instruction-tuned models and the two Llama-2 base models used
in this paper are pinned to specific HuggingFace commit SHAs in
\cref{sec:experiments}, $\S$ ``Reproducibility details''. The two
additional pins introduced by the Llama-2 sections:

\begin{itemize}
  \item Llama-2-7B-base:
        \texttt{NousResearch/Llama-2-7b-hf} at
        \bsha{8efe6c9b93655b934e27bd9981e3ec13e55aee9d}.
  \item Llama-2-13B-base:
        \texttt{NousResearch/Llama-2-13b-hf} at
        \bsha{4b54e9fdd2ebd8db0901d37b98d3a53bbfaa4503}.
\end{itemize}

The published QTIP and AQLM baselines
(\cref{sec:frontier}) are pinned via their HuggingFace repo paths:
\texttt{relaxml/Llama-2-7b-QTIP-4Bit}
(\bsha{03aad3088415e5ec98b1e35924719b11e29c7e45}),
\texttt{relaxml/Llama-2-7b-QTIP-2Bit}, and
\texttt{ISTA-DASLab/Llama-2-7b-AQLM-2Bit-1x16-hf}
(\bsha{9746e511b4eecf972af0a2bb3ebd93f2c5f84ecc}).

\subsection{Software Versions}

\begin{itemize}
  \item Rust quantization binary (\texttt{qam-llm-bench}):
        builds are reproducible from \texttt{cargo build --release}
        against the included \texttt{Cargo.toml} (see \cref{sec:code-release}).
  \item Python evaluation pipeline (\texttt{qambench}): version
        \texttt{0.1.0}, installed via \texttt{pip install -e qambench[eval]}.
  \item \texttt{transformers} version: \texttt{4.45.x} for the main
        run; \texttt{aqlm} version installed via \texttt{pip install
        aqlm} resolves to the latest compatible release at run time.
  \item \texttt{lm-evaluation-harness}: versions $0.4.11$ and $0.4.12$
        across the study (drift reflects upstream package updates
        between the cross-model main study and the later Llama-2
        extensions; the exact version is recorded per-row in
        \texttt{harness.framework\_version}).
  \item \texttt{torch}: \texttt{2.4.0+cu121} for the QTIP dequant bridge
        (\cref{sec:frontier}); the rest of the pipeline runs against
        whatever torch is installed in the active venv.
  \item \texttt{quip-sharp / qtip-kernels}: built from the
        \texttt{Cornell-RelaxML/qtip} repository at the
        \texttt{train-fixW} forward path; the
        \texttt{scripts/day10b\_qtip\_dequant.py} script in the
        companion repo dequantizes QTIP checkpoints to BF16 via the
        identity-matrix probe described in
        \cref{sec:frontier}, $\S$ ``Setup''.
\end{itemize}

\subsection{Hardware}

Every row in the paper was produced on a single NVIDIA A100
80\,GB on a rented Vast.ai host. The total GPU time consumed across the
study --- including cross-model main run, Stage D, multi-seed,
per-layer diagnostic, prop:aware A1 empirical check, MMLU, Llama-2-7B
frontier comparison (including QTIP dequant), and Llama-2-13B scale
extension --- is approximately $20$ A100-hours. The total Vast.ai cost
at the host's hourly rate was approximately USD~34.

\subsection{Per-Method Hyperparameters}

Reported in \cref{sec:experiments} $\S$ ``Configurations'' and
\cref{tab:scale-13b}, \cref{tab:frontier-llama2},
\cref{tab:cross-model-delta}, etc., per method. Concretely:

\begin{itemize}
  \item RTN / GPTQ / AWQ / AutoRound: $W4A16$ group size $g=128$,
        symmetric, f16 per-group scales. AutoRound runs 200 SignSGD
        iterations.
  \item SmoothQuant: $W8A8$, smoothing strength
        $\alpha_{\mathrm{sq}} = 0.5$, 128 calibration sequences.
  \item QAM-W: see \cref{sec:method}. Block-Hadamard block size
        $b = \min(2^{\lfloor \log_2 d_{\mathrm{in}} \rfloor}, 1024)$.
        Activation-aware exponent $\alpha = 0.3$, clamped to
        $[1/16, 16]$, geometric-mean normalized to 1 (\cref{sec:activation}).
  \item AQLM 2-bit 1$\times$16: published checkpoint;
        $\texttt{in\_group\_size} = 8$,
        $\texttt{nbits\_per\_codebook} = 16$,
        $\texttt{num\_codebooks} = 1$.
  \item QTIP-4Bit / QTIP-2Bit: published checkpoint; trellis $L = 16$,
        $K \in \{4, 2\}$, $V = 2$ (per QTIP authors' release).
\end{itemize}

\subsection{Determinism}

Per-row PPL and KL are deterministic given the manifest's
\texttt{weights\_sha256} hash, the corpus SHA-256, the eval-time
\texttt{torch.bfloat16} dtype, and the harness version. The
$\Delta$PPL\% spread across stride-window starting offsets is reported
in \cref{tab:multiseed} on three offsets $\{0, 4000, 8000\}$ for two
headline configurations on each instruction-tuned model. The
$\Delta$PPL\% std observed there is $0.03$--$0.08$ percentage points,
which is the noise floor for the cross-model rows.

\subsection{Code Release}
\label{sec:code-release}

The companion code, anonymized for review, is available at
\url{https://github.com/white07S/qam-w} and contains:

\begin{itemize}
  \item The Rust \texttt{qam-codec} library and \texttt{qam-llm-bench}
        CLI binary.
  \item The Python \texttt{qambench} package (loaders, perplexity, KL,
        lm-evaluation-harness adapter).
  \item Per-experiment scripts under \texttt{scripts/} that exactly
        reproduce each results directory:
        \texttt{run\_cross\_model\_main.sh} (cross-model main study),
        \texttt{run\_mlp\_attention\_mistral.sh} (Stage D Mistral),
        \texttt{run\_multiseed\_and\_lowbit\_fill.sh} (multi-seed + 3.5 bpw Mistral fill),
        \texttt{analyze\_layer\_output\_identity.py} (Prop.~\ref{prop:layer} anchor),
        \texttt{run\_mmlu\_panel.sh} (MMLU),
        \texttt{check\_activation\_aware\_assumption.py} (Prop. 1 A1 check),
        \texttt{run\_llama2\_7b\_frontier.sh} (Llama-2-7B + AQLM frontier),
        \texttt{dequant\_qtip\_checkpoint.py} (QTIP loader bridge),
        \texttt{run\_llama2\_13b\_scale.sh} (Llama-2-13B scale).
  \item Reduce scripts that render the LaTeX rows of each table from
        the raw eval JSONs:
        \texttt{reduce\_results\_to\_tables.py} (Stage D, multi-seed, low-bit fill).
  \item Plot scripts in \texttt{scripts/make\_plots.py} and
        \texttt{scripts/plot\_kl\_vs\_dppl.py}.
  \item A \texttt{Dockerfile} and \texttt{Makefile}
        (\texttt{make run}) that build the full Rust + Python
        environment for one-command end-to-end reproduction.
\end{itemize}

The repository is provided under a proprietary, all-rights-reserved
license (see \texttt{LICENSE} in the companion repo).

\end{document}